%% file: slod_arxiv_v2.tex
\documentclass{ceurart}
\usepackage{amssymb, amsthm, mathtools}
\usepackage{bm}

\usepackage{booktabs, caption, subcaption}
\usepackage[table]{xcolor}
\usepackage{float}

\usepackage{algorithm, algorithmic}

\usepackage{enumitem}

\newtheorem{definition}{Definition}
\newtheorem{theorem}{Theorem}
\newtheorem{proposition}{Proposition}

\DeclareMathOperator{\dH}{d_{\mathbb{H}}}
\DeclareMathOperator{\Exp}{Exp}
\DeclareMathOperator{\Log}{Log}
\DeclareMathOperator{\JSD}{JSD}
\DeclareMathOperator*{\argmin}{arg\,min}
\newcommand{\BB}{\mathbb{B}^d}
\newcommand{\RR}{\mathbb{R}}
\newcommand{\HH}{\mathbb{H}}

\begin{document}

\copyrightyear{2026}
\copyrightclause{Copyright for this paper by its authors.
  Use permitted under Creative Commons License Attribution 4.0
  International (CC BY 4.0).}

\conference{1st Workshop on Graphs Across AI: From Structural Reasoning to
  Foundation Models, held in conjunction with the IEEE World Congress on
  Computational Intelligence (WCCI 2026), June 21--26, 2026, Maastricht,
  The Netherlands}

\title{Semantic Level of Detail for Knowledge Graphs: Discovering
  Abstraction Boundaries via Spectral Heat Diffusion}

\tnotemark[1]
\tnotetext[1]{Discussion paper presented at the 1st Workshop on Graphs Across AI (GRAAI), IEEE WCCI 2026; this extended manuscript adds the full proofs of Lemmas~A.1 and~A.2 in Appendix~\ref{app:coherence}.}

\author[1]{Edward Izgorodin}[%
  orcid=0009-0006-2529-4393,
  email=izgorodin@me.com,
]
\cormark[1]
\address[1]{Mnemoverse.AI, Funchal, Madeira, Portugal}
\cortext[1]{Corresponding author.}

\begin{abstract}
Graph-structured knowledge systems---from knowledge graphs to GraphRAG pipelines---increasingly organize information into hierarchical communities, yet lack a principled mechanism for \emph{continuous resolution control}: where do the qualitative boundaries between abstraction levels lie, and how should an agent navigate them? Current approaches rely on discrete community detection with manually tuned resolution parameters (e.g., Leiden~$\gamma$), offering no continuous zoom and no formal guarantees. We introduce \emph{Semantic Level of Detail} (SLoD), a framework that addresses both problems by defining a continuous zoom operator via heat kernel diffusion on a graph Laplacian whose kNN structure is induced by a Poincar\'{e}-ball embedding. We prove hierarchical coherence in the tree limit (exact tree with Sarkar embedding), with bounded approximation error, and demonstrate consistent boundary-detection behaviour on noisy hierarchies; spectral gaps in the graph Laplacian then induce \emph{emergent scale boundaries}---scales where the representation undergoes qualitative transitions---detectable without manual resolution tuning. On synthetic hierarchies (HSBM, 1024 nodes), spectral clustering at the BoundaryScan-detected scale recovers planted levels, with macro ARI saturating at $1.00$ in the high-SNR regime (50-seed median, Table~\ref{tab:ablation_macro}) and meso ARI reaching $0.89$ [$0.86$, $0.92$] at $r{=}200$. On the full WordNet noun hierarchy (82K synsets), using 100 stratified leaf queries, detected boundaries align with true taxonomic depth ($\tau = 0.79$), demonstrating meaningful abstraction-level discovery in real-world knowledge graphs without resolution-parameter tuning. The composite weights, MAD threshold, and kNN-parameter rule ($k = \max(10, \min(\lfloor\sqrt{N}\rfloor, 50))$) use defaults that transferred unchanged between HSBM and WordNet; their behavior on graphs with implicit or qualitatively different hierarchical structure is open (\S\ref{sec:discussion}).
\end{abstract}


\maketitle

\hypersetup{pdfauthor={Edward Izgorodin}}

\section{Introduction}\label{sec:intro}

Graph-based knowledge organization is central to modern AI systems. GraphRAG \cite{edge2024graphrag}, knowledge graphs, and persistent memory systems like MemGPT \cite{packer2023memgpt} organize information into entity graphs with community structure. Yet a fundamental question remains open: \emph{where do the meaningful abstraction boundaries lie within a knowledge graph, and how should an agent transition between them?}

Current graph-based systems rely on discrete community detection with manually tuned resolution parameters. Leiden clustering requires sweeping over $\gamma$ values with no guarantee that the chosen granularity matches the task, and it produces discrete partitions rather than a continuous zoom. A software project has architecture-level concepts, module-level patterns, and line-level details; an agent reasoning over such a graph needs to dynamically select the appropriate resolution---a capability that existing systems lack.

We draw inspiration from computer graphics, where \emph{Level of Detail} (LOD) allows rendering engines to represent geometry at variable resolution \cite{luebke2003lod}. \textbf{Can we build an analogous LOD operator for knowledge graphs?}

Our key insight is that hyperbolic space provides the natural substrate. The Poincar\'{e} ball $\BB$ embeds tree-structured hierarchies with distortion $(1{+}\varepsilon)$ for any $\varepsilon > 0$ \cite{sarkar2011low}, while heat kernel diffusion on the graph Laplacian provides a family of smoothing operators parameterized by a continuous scale $\sigma$. Spectral gaps in the Laplacian then induce \emph{emergent scale boundaries}---values of $\sigma$ where the representation undergoes qualitative transitions---connecting SLoD to multi-scale community detection \cite{delvenne2010stability}.

\paragraph{Contributions.} \textbf{(1)}~A mathematical formulation of Semantic LOD as heat kernel diffusion on the Poincar\'{e} ball; \textbf{(2)}~Theoretical guarantees on hierarchical coherence and approximation quality; \textbf{(3)}~An efficient tangent-space aggregation algorithm; \textbf{(4)}~An emergent scale selection procedure tied to spectral structure; \textbf{(5)}~Empirical validation on synthetic hierarchies (HSBM) and a real-world DAG (WordNet, 82K synsets).\footnote{Code and experiment artefacts: \url{https://github.com/mnemoverse/mnemoverse-slod-paper}; permanent archive: \url{https://doi.org/10.5281/zenodo.19933482}}

\section{Background}\label{sec:prelim}

\paragraph{Heat Kernel on Graphs and Hyperbolic Space.}
The heat kernel $K_\sigma(x,y)$ on $\HH^d$ is the fundamental solution of the heat equation $\partial_\sigma u = \Delta_{\HH} u$. For $\HH^d$ with curvature $\kappa{=}{-}1$ \cite{grigoryan2009heat}:
\begin{equation}\label{eq:heat_kernel}
  K_\sigma(x,y) = \frac{1}{(4\pi\sigma)^{d/2}} e^{-(d{-}1)^2\sigma/4} \int_{\dH(x,y)}^{\infty} \frac{s \, e^{-s^2/(4\sigma)}}{(\cosh s - \cosh \dH(x,y))^{1/2}} \, ds
\end{equation}
The key property: $K_\sigma$ depends only on geodesic distance (isotropy), with $\sigma \to 0$ recovering point detail and $\sigma \to \infty$ producing uniform averaging. This makes $\sigma$ a natural ``zoom'' parameter that progressively simplifies representations without introducing spurious detail \cite{lindeberg1994scale}.

\paragraph{Fr\'{e}chet Mean on the Poincar\'{e} Ball.}
The Poincar\'{e} ball $\BB = \{x \in \RR^d : \|x\| < 1\}$ is a Hadamard manifold where the Fr\'{e}chet functional $F(z) = \sum_i w_i \, \dH^2(z, v_i)$ is strictly convex, ensuring a unique minimizer \cite{sturm2003probability, afsari2011riemannian}. Unlike Euclidean or spherical spaces, the Fr\'{e}chet mean on $\BB$ never bifurcates. Scale boundaries therefore manifest as rapid displacement of the mean and sensitivity changes in the Hessian---an observation that shapes our boundary detection approach.

\section{Semantic Level of Detail}\label{sec:method}

\subsection{Problem Formulation}

Given a knowledge corpus embedded as $\mathcal{V} = \{v_1, \ldots, v_N\} \subset \BB$, we seek a family of operators $\Phi_\sigma: 2^{\BB} \to \BB$ such that:
\begin{enumerate}
  \item \textbf{Coarse scale} ($\sigma$ large): $\Phi_\sigma(\mathcal{V})$ captures the global semantic theme.
  \item \textbf{Fine scale} ($\sigma$ small): $\Phi_\sigma(\mathcal{V})$ preserves local semantic detail.
  \item \textbf{Continuity}: The map $\sigma \mapsto \Phi_\sigma(\mathcal{V})$ is smooth.
  \item \textbf{Hierarchy-aware}: Nearby scales produce semantically related representations.
\end{enumerate}

\subsection{Diffusion-Based Aggregation}

\begin{definition}[Heat Kernel Weights]\label{def:weights}
For focus point $x_0 \in \BB$ and scale $\sigma > 0$:
\begin{equation}\label{eq:weights}
  w_i(\sigma, x_0) = \frac{K_\sigma(x_0, v_i)}{\sum_{j=1}^N K_\sigma(x_0, v_j)}
\end{equation}
\end{definition}

\begin{definition}[Semantic LOD Operator]\label{def:slod}
The SLoD representation at scale $\sigma$ with focus $x_0$ is the weighted Fr\'{e}chet mean:
\begin{equation}\label{eq:frechet}
  \Phi_\sigma(\mathcal{V}, x_0) = \argmin_{y \in \BB} \sum_{i=1}^N w_i(\sigma, x_0) \cdot \dH^2(y, v_i)
\end{equation}
\end{definition}

Since the Fr\'{e}chet mean has no closed form on $\BB$, we compute it iteratively in tangent space:

\begin{algorithm}[H]
\caption{SLoD via Tangent-Space Aggregation}
\label{alg:slod}
\begin{algorithmic}[1]
\REQUIRE Embeddings $\mathcal{V} = \{v_i\}_{i=1}^N \subset \BB$, focus $x_0 \in \BB$, scale $\sigma$, max iterations $T$, step size $\eta$ (default $1.0$ for Riemannian gradient descent on the Poincar\'{e} ball), convergence tolerance $\mathrm{tol}$
\STATE Compute weights $w_i \gets K_\sigma(x_0, v_i) / \sum_j K_\sigma(x_0, v_j)$
\STATE Initialize $\mu^{(0)} \gets v_{i^*}$ where $i^* = \arg\max_i w_i$ \hfill\COMMENT{Warm-start at the dominant-weight point}
\FOR{$t = 0$ to $T-1$}
    \STATE $u_i \gets \Log_{\mu^{(t)}}(v_i)$ for all $i$ \hfill\COMMENT{Map to tangent space}
    \STATE $\bar{u} \gets \sum_i w_i \cdot u_i$ \hfill\COMMENT{Weighted average}
    \STATE $\mu^{(t+1)} \gets \Exp_{\mu^{(t)}}(\eta \cdot \bar{u})$ \hfill\COMMENT{Map back}
    \IF{$\dH(\mu^{(t+1)}, \mu^{(t)}) < \mathrm{tol}$} \STATE \textbf{break} \hfill\COMMENT{Early exit on convergence} \ENDIF
\ENDFOR
\RETURN $\mu^{(t+1)}$
\end{algorithmic}
\end{algorithm}

For large $N$, we approximate via the graph Laplacian: construct a $k$-NN graph with hyperbolic distance weights, compute the normalized Laplacian $L$, and evaluate $\mathbf{w}(\sigma) = \exp(-\sigma L) \, \mathbf{e}_{x_0}$ via spectral methods \cite{hammond2011wavelets}; in practice, we use Lanczos partial eigendecomposition (\texttt{scipy.sparse.linalg.eigsh}) retaining $K_{\mathrm{eigs}}$ eigenpairs, costing $O(N K_{\mathrm{eigs}}^2)$, linear in $N$ (we distinguish $K_{\mathrm{eigs}}$, the Lanczos eigenvalue count, from the planted-level count $K$ used to describe HSBM hierarchy levels and from the data-dependent effective dimensionality $K^*(\sigma) = \min\{k : \sum_{i=1}^{k} e^{-\sigma\lambda_i} \big/ \sum_{i=1}^{N} e^{-\sigma\lambda_i} \ge 0.95\}$, the smallest number of heat-decayed eigenmodes capturing $95\%$ of the spectral energy at scale $\sigma$; this is the integer-valued effective rank used as the $y$-axis of Figure~\ref{fig:kstar_trajectory} and the return value of Algorithm~\ref{alg:boundary}). \emph{All experiments in this paper use this graph-Laplacian implementation; the continuous kernel of Eq.~\eqref{eq:heat_kernel} serves as the geometric idealization motivating the construction.} Empirically (\S\ref{sec:ablation}, takeaway~(a); \S\ref{sec:discussion}, \emph{Where the Hyperbolic Geometry Enters}), the hyperbolic substrate enters this pipeline primarily through the kNN graph Laplacian; the Fr\'echet step of Definition~\ref{def:slod} then returns the principled Poincar\'e-ball summary at the selected scale, with coherence guaranteed along the trajectory by Theorem~\ref{thm:coherence} below.

\section{Theoretical Analysis}\label{sec:theory}

\begin{theorem}[Hierarchical Coherence]\label{thm:coherence}
Let $\mathcal{T}$ be a weighted tree of bounded maximum degree, embedded in $\BB$ via Sarkar embedding with distortion $\delta = (1{+}\varepsilon)$, and let $L$ be the symmetric normalized Laplacian of the kNN graph on the embedded points. For any node $v$ with subtree-edge-diameter $D_{\mathcal{T}}(v)$ and $\sigma_1 < \sigma_2$:
\begin{equation}
  \dH\!\left(\Phi_{\sigma_1}(\mathcal{V}_v, v), \Phi_{\sigma_2}(\mathcal{V}_v, v)\right) \leq C \cdot |\sigma_2 - \sigma_1| \cdot (1{+}\varepsilon),
\end{equation}
where $C = \ell\,\|L\|_{1\to 1}\,D_{\mathcal{T}}(v)$ factors into the Sarkar base edge length $\ell$, the kNN-graph conditioning $\|L\|_{1\to 1} := \max_j \sum_i |L_{ij}|$, and the subtree-edge diameter. $C$ is independent of $N$ for fixed subtree depth and bounded kNN degree ratio (full statement and proof in Appendix~\ref{app:coherence}).
\end{theorem}

The proof composes two Lipschitz steps: the Fr\'{e}chet mean is $D/2$-Lipschitz in its weights on every Hadamard manifold \cite{karcher1977riemannian,sturm2003probability,afsari2011riemannian}, and the heat-kernel weights are $2\,\|L\|_{1\to 1}$-Lipschitz in $\sigma$ on the simplex. See Appendix~\ref{app:coherence} for the full argument with explicit constants.

In the local small-scale regime, the scale-dependent approximation error is $O(\sigma)$: as $\sigma \to 0$, the heat kernel converges to a delta distribution, and the error is controlled by the variance of the weight distribution (full assumptions in Appendix~\ref{app:coherence}).

\paragraph{Why Hyperbolic?} The Poincar\'{e} ball's contribution to SLoD is concentrated in one upstream step: kNN-graph construction. Hyperbolic distance defines neighbourhoods that respect tree-like structure---the Poincar\'{e} ball achieves Sarkar $(1{+}\varepsilon)$-distortion tree embedding in any dimension $d \ge 2$ \cite{sarkar2011low}, while Euclidean $\RR^d$ embeds the same tree with at best $\widetilde{O}(\Lambda^{1/(d-1)})$ distortion in $d \ge 2$ \cite{gupta2000embedding} (the symbol $\Lambda$ denotes the number of tree leaves, to avoid collision with the Laplacian $L$), nearly matching the prior lower bound $\Omega(\Lambda^{1/d})$. Empirically, replacing Poincar\'{e}-distance kNN with Euclidean-distance kNN costs $5.7$\,pp meso ARI aggregated across $r{\in}\{80,100,150,200\}$ (\S\ref{sec:ablation} takeaway~(b); Wilcoxon $p < 10^{-15}$). Once the Laplacian is built, the Fr\'{e}chet-centroid trajectory $\sigma \mapsto \Phi_\sigma$ is downstream of this geometric choice and is governed by the coherence guarantee of Theorem~\ref{thm:coherence} \emph{under correct hyperbolic coordinates}; as the random-points control demonstrates (\S\ref{sec:ablation} takeaway~(a)), the eigenvector-based $\sigma^*$ selection is robust to substituting i.i.d.\ Gaussian noise for the Poincar\'{e} coordinates, but this does not validate $\Phi_\sigma$ equivalence under random aggregation. Hyperbolic geometry thus enters as a topology-defining metric for the kNN graph, not as a coordinate-system requirement for $\sigma^*$ selection (the latter is structural: $\sigma^*$ is read off Laplacian eigenvectors, which depend on the graph but not on the point cloud, so the random-points equivalence at $\sigma^*$ is not just an empirical observation but follows by construction).

\section{Emergent Scale Selection}\label{sec:boundary}

The SLoD operator provides representations at any scale $\sigma$, but practical systems need to identify \emph{which scales matter}. We show that the spectral structure of the graph Laplacian induces natural boundaries.

\subsection{Spectral Motivation}

The heat kernel's spectral decomposition on $L$ with eigenvalues $0 = \lambda_1 \leq \lambda_2 \leq \cdots \leq \lambda_N$:
\begin{equation}\label{eq:spectral_decomp}
  K_\sigma(i,j) = \sum_{k=1}^{N} e^{-\sigma \lambda_k} \phi_k(i) \phi_k(j)
\end{equation}
At diffusion time $\sigma$, modes with $\lambda_k \sigma \gg 1$ are exponentially suppressed. A \emph{spectral gap} between $\lambda_k$ and $\lambda_{k+1}$ creates a range $1/\lambda_{k+1} < \sigma < 1/\lambda_k$ where exactly $k$ modes dominate.

\begin{proposition}[Spectral Scale Boundaries, heuristic; cf.\ \cite{vonluxburg2007tutorial, hammond2011wavelets}]\label{prop:spectral}
Under isolated-gap and mode-dominance assumptions, let $L$ have gap ratio $\rho_k = \lambda_{k+1}/\lambda_k > R$ for threshold $R > 1$ (distinct from the HSBM signal ratio $r$ of \S\ref{sec:exp1}). Then:
\begin{itemize}
  \item[\textbf{(i)}] The weight divergence $D_w(\sigma) = \JSD(\mathbf{w}(\sigma) \| \mathbf{w}(\sigma{+}\Delta\sigma))$ achieves a local maximum near $\sigma^* \approx 1/\lambda_k$.
  \item[\textbf{(ii)}] The effective dimensionality drops from $k$ to $k{-}1$ as $\sigma$ crosses $\sigma^*$.
  \item[\textbf{(iii)}] The representation velocity satisfies $V(\sigma^*) \geq c\,(\lambda_{k+1} - \lambda_k)$ for some graph-dependent constant $c > 0$ (a velocity floor proportional to the absolute spectral gap).\footnote{The experiments in this paper probe (i) and (ii) via composite-peak alignment with $1/\lambda_k$ and via the integer-valued $K^*(\sigma)$ trajectory through planted levels (Figure~\ref{fig:kstar_trajectory}); empirical verification of the (iii) velocity-floor bound is left to future work.}
\end{itemize}
\end{proposition}

\noindent (i)--(iii) follow from the spectral decomposition (Eq.~\ref{eq:spectral_decomp}): when $\lambda_k \sigma \!\sim\! 1$, modes with $\lambda_{k+1} \sigma \gg 1$ are suppressed while modes $\le k$ remain active, so a spectral-gap crossing produces a JSD bump (i), a discrete drop in effective dimensionality (ii), and a velocity floor proportional to the gap (iii); analogous results appear in spectral-clustering theory \cite{vonluxburg2007tutorial} and graph-wavelet analysis \cite{hammond2011wavelets}.

\begin{figure}[h]
\centering
\includegraphics[width=0.85\columnwidth]{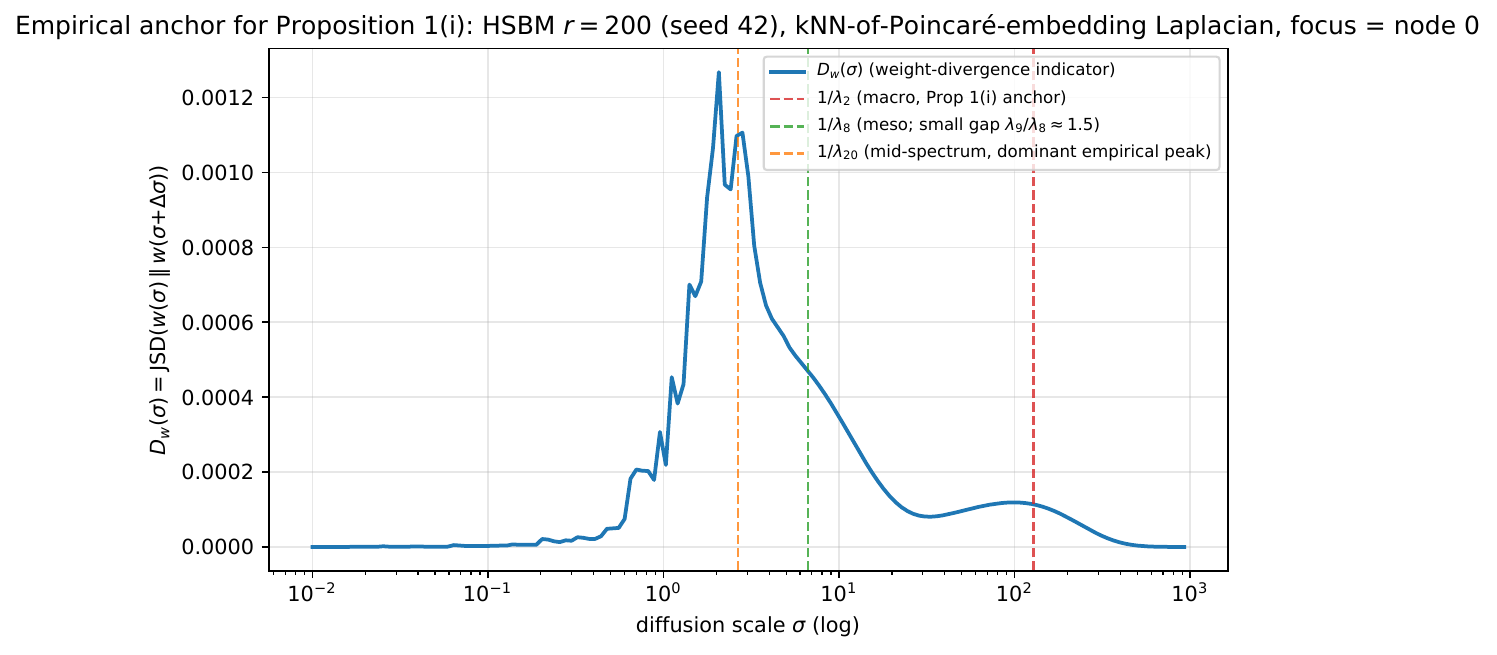}
\caption{Empirical anchor for Proposition~\ref{prop:spectral}(i) on the kNN-of-Poincar\'e-embedding Laplacian (HSBM, $r{=}200$, seed~42, focus = node~0). The weight-divergence indicator $D_w(\sigma) = \JSD(\mathbf{w}(\sigma)\,\|\,\mathbf{w}(\sigma{+}\Delta\sigma))$ exhibits two empirical features: a secondary local maximum near $\sigma \approx 1/\lambda_2 \approx 128$ (the macro Fiedler scale, where the gap ratio $\lambda_3/\lambda_2 \approx 9.8$ is large and Proposition~\ref{prop:spectral}(i) predicts a peak), and a dominant peak in the mid-spectrum region near $1/\lambda_{20}$ where many modes co-reorganise simultaneously. The meso scale $1/\lambda_8$ does not produce a clean local maximum, consistent with the small gap there ($\lambda_9/\lambda_8 \approx 1.5$, below the $R{=}2$ threshold used in Algorithm~\ref{alg:boundary}); Proposition~\ref{prop:spectral}(i)'s prediction is conditional on a large gap, and the meso level at this regime falls below that conditioning threshold.}
\label{fig:dw_anchor}
\end{figure}

\subsection{Boundary Detection Algorithm}

We combine three complementary signals into a composite boundary score:

\begin{definition}[Boundary Indicators]\label{def:indicators}
For scale $\sigma$ with step $\Delta\sigma$ (distinct from the Sarkar distortion $\delta$ of \S\ref{sec:theory}):
\begin{align}
  V(\sigma) &= \frac{\dH(\Phi_{\sigma+\Delta\sigma}, \Phi_\sigma)}{\Delta\sigma} \quad \text{(representation velocity)} \label{eq:rep_vel}\\
  D_w(\sigma) &= \JSD(\mathbf{w}(\sigma) \| \mathbf{w}(\sigma{+}\Delta\sigma)) \quad \text{(weight divergence)} \label{eq:weight_div}\\
  C_k(\sigma) &= 1 - \frac{|NN_k(\sigma) \cap NN_k(\sigma{+}\Delta\sigma)|}{|NN_k(\sigma) \cup NN_k(\sigma{+}\Delta\sigma)|} \quad \text{(neighborhood churn)} \label{eq:churn}
\end{align}
\end{definition}

The three indicators capture complementary aspects of a scale crossing: $V$ tracks geometric drift of the Fr\'echet centroid in $\BB$, $D_w$ tracks distributional shift of the diffusion weights, and $C_k$ tracks topological reorganisation of the local neighbourhood. They are z-score normalised ($\hat{V}, \hat{D}, \hat{C}$) and combined into a composite score $S(\sigma) = \alpha_1\hat{V} + \alpha_2\hat{D} + \alpha_3\hat{C}$ whose peaks mark candidate boundaries. Algorithm~\ref{alg:boundary} formalises the procedure.

\begin{algorithm}[H]
\caption{SLoD-BoundaryScan}
\label{alg:boundary}
\begin{algorithmic}[1]
\REQUIRE Embeddings $\mathcal{V} \subset \BB$, focus $x_0$, scale grid $\Sigma$ (log-spaced), composite weights $(\alpha_1, \alpha_2, \alpha_3)$, MAD multiplier $\beta$
\ENSURE Boundary set $\Sigma^*$, effective dimensionalities $\{K^*(\sigma^*)\}$

\STATE Compute graph Laplacian $L$ and eigenvalues $\lambda_1, \ldots, \lambda_K$
\STATE $\mathcal{C} \gets \{1/\lambda_k : \lambda_{k+1}/\lambda_k > R\}$ \hfill\COMMENT{Candidate scales}

\FOR{each $\sigma \in \Sigma$}
  \STATE $\mathbf{w}_\sigma \gets$ DiffusionWeights$(\mathcal{V}, x_0, \sigma)$; \quad $\Phi_\sigma \gets$ Algorithm~\ref{alg:slod} \hfill\COMMENT{Fr\'echet mean at $\sigma$}
\ENDFOR

\FOR{consecutive $(\sigma_t, \sigma_{t+1})$}
  \STATE Compute $V_t$, $D_t$, $C_t$ via Definition~\ref{def:indicators} (using $\Phi_\sigma, \mathbf{w}_\sigma$)
\ENDFOR

\STATE $S_t \gets \alpha_1 \hat{V}_t + \alpha_2 \hat{D}_t + \alpha_3 \hat{C}_t$ \hfill\COMMENT{Composite score}
\STATE $\Sigma^* \gets$ PeakPick$(S, \text{median} + \beta \cdot \text{MAD})$ \hfill\COMMENT{$\beta$ is the MAD multiplier; distinct from composite weights $\alpha_1,\alpha_2,\alpha_3$}
\RETURN $\Sigma^*, \{K^*(\sigma^*)\}$
\end{algorithmic}
\end{algorithm}

When the effective dimensionality $K^*(\sigma) > 1$ at a boundary, a single Fr\'{e}chet mean is a poor summary. We extend SLoD to a mixture representation via weighted Riemannian $k$-means, with each cluster center achieving Sarkar distortion over at most $N/K$ nodes.

\section{Experiments}\label{sec:experiments}

\paragraph{Evaluation Metrics.}
We use metrics standard for their respective tasks.
\textbf{ARI} (Adjusted Rand Index) \cite{hubert1985comparing}: chance-corrected agreement between predicted and planted partitions; $\mathrm{ARI}{=}1$ perfect, $0$ random, negative below chance.
\textbf{Kendall $\tau$} \cite{kendall1938new}: rank correlation between detected boundary scales $\sigma^*$ and the \emph{heights} of their nearest true ancestor depths, where each detected boundary is matched to the closest true ancestor depth and the matched depth is converted to height as $(\max\text{depth}-\text{depth})$; $\tau \in [-1, 1]$, expected positive (larger $\sigma$ $\Rightarrow$ coarser $\Rightarrow$ shallower ancestor). \textbf{Precision@}$h$: fraction of detected boundary depths within $h$ hierarchy levels of any true ancestor depth (reported at $h{=}0.5$). \textbf{Recall@}$h$: fraction of true ancestor depths covered by at least one detected boundary within $h$ levels (reported at $h{=}1, 2$). Detected boundary depths are obtained by matching each detected $\sigma^*$ to the closest true ancestor along the focus-to-root path and recording that ancestor's integer depth, so Precision@$0.5$ reduces to exact-match precision on integer depths; Recall@$1$ and Recall@$2$ then admit one-level and two-level slack respectively.
\textbf{HP} (Hierarchy Preservation) and \textbf{SP} (Sibling Proximity) are Poincar\'e embedding quality gates we adopt for this work (analogous in spirit to the rank-based reconstruction metrics evaluated for Poincar\'e embeddings in \cite{nickel2017poincare}, but operationalised here as ratio gates rather than mean rank): HP is the fraction of parent--child pairs with $\dH(\text{parent}, \text{child}) < \dH(\text{parent}, \text{random})$ (gate $> 0.90$); SP is the analogous condition for sibling vs.\ random pairs (gate $> 0.85$).

\subsection{Experiment 1: Boundary Recovery on Synthetic Hierarchies}\label{sec:exp1}

\paragraph{Setup.} Hierarchical Stochastic Block Model (HSBM) with planted 3-level structure: 1024 nodes $\to$ 2 macro $\to$ 8 meso $\to$ 64 micro communities. We vary inter-level ratio $r \in \{20, 40, 60, 80, 100, 150, 200\}$ to test detection sensitivity across the information-theoretic Kesten--Stigum phase transition. We compare against Louvain, greedy modularity, spectral $k$-sweep \cite{vonluxburg2007tutorial}, and the Leiden algorithm \cite{traag2019louvain} run with the Constant Potts Model resolution \cite{traag2011narrow} (henceforth ``Leiden CPM''). Hyperparameters use the unified default rule $k = \max(10, \min(\lfloor\sqrt{N}\rfloor, 50))$ for the kNN graph, giving $k{=}32$ for HSBM ($N{=}1024$) and $k{=}50$ for WordNet ($N{\approx}82\mathrm{K}$, \S\ref{sec:exp2}); composite weights $(\alpha_1, \alpha_2, \alpha_3) = (1/3, 1/3, 1/3)$; MAD multiplier $\beta = 2$; gap threshold $R = 2$ (Algorithm~\ref{alg:boundary}); Algorithm~\ref{alg:slod} runs at most $T{=}15$ iterations with step $\eta{=}1.0$ and early exit at convergence tolerance $\mathrm{tol}{=}10^{-6}$ on $\dH(\mu^{(t+1)}, \mu^{(t)})$; Lanczos eigendecomposition uses \texttt{scipy.sparse.linalg.eigsh} (scipy~$\ge$~1.12) at default tolerance.

\paragraph{Results.} Table~\ref{tab:exp1_ari} reports the \emph{direct-Laplacian baseline} (no embedding step) to localize SLoD's incremental contribution; the full kNN-of-embedding pipeline (the actual SLoD method) is evaluated under multi-seed conditions in \S\ref{sec:ablation}.

\begin{table}[h]
\centering
\caption{Adjusted Rand Index at macro ($K^*{=}2$), meso ($K^*{=}8$), and micro ($K^*{=}64$) scales across signal-to-noise ratios. Numbers are from spectral clustering on the direct combinatorial graph Laplacian with $K_{\text{eigs}}{=}80$ (sufficient to resolve the planted micro level); single seed (42), with multi-seed evidence in \S\ref{sec:ablation}. The kNN-of-Poincar\'e-embedding default of \S\ref{sec:ablation} reaches macro ARI ${=}0.42$ at $r{=}20$ (Table~\ref{tab:ablation_macro}), so the low \emph{macro} values here at $r{\le}40$ reflect the spectral-clustering baseline rather than SLoD itself (see takeaway~(a) of \S\ref{sec:ablation}).}
\label{tab:exp1_ari}
\small
\begin{tabular}{@{}lccccccc@{}}
\toprule
$r$ & 20 & 40 & 60 & 80 & 100 & 150 & 200 \\
\midrule
ARI (macro, $K^*{=}2$)  & 0.04 & 0.67 & 0.84 & 0.93 & 0.97 & 1.00 & 1.00 \\
ARI (meso, $K^*{=}8$)   & 0.09 & 0.15 & 0.24 & 0.43 & 0.63 & 0.79 & 0.91 \\
ARI (micro, $K^*{=}64$) & 0.04 & 0.10 & 0.24 & 0.47 & 0.69 & 0.91 & 1.00 \\
\bottomrule
\end{tabular}
\end{table}

\begin{figure}[t]
\centering
\includegraphics[width=\columnwidth]{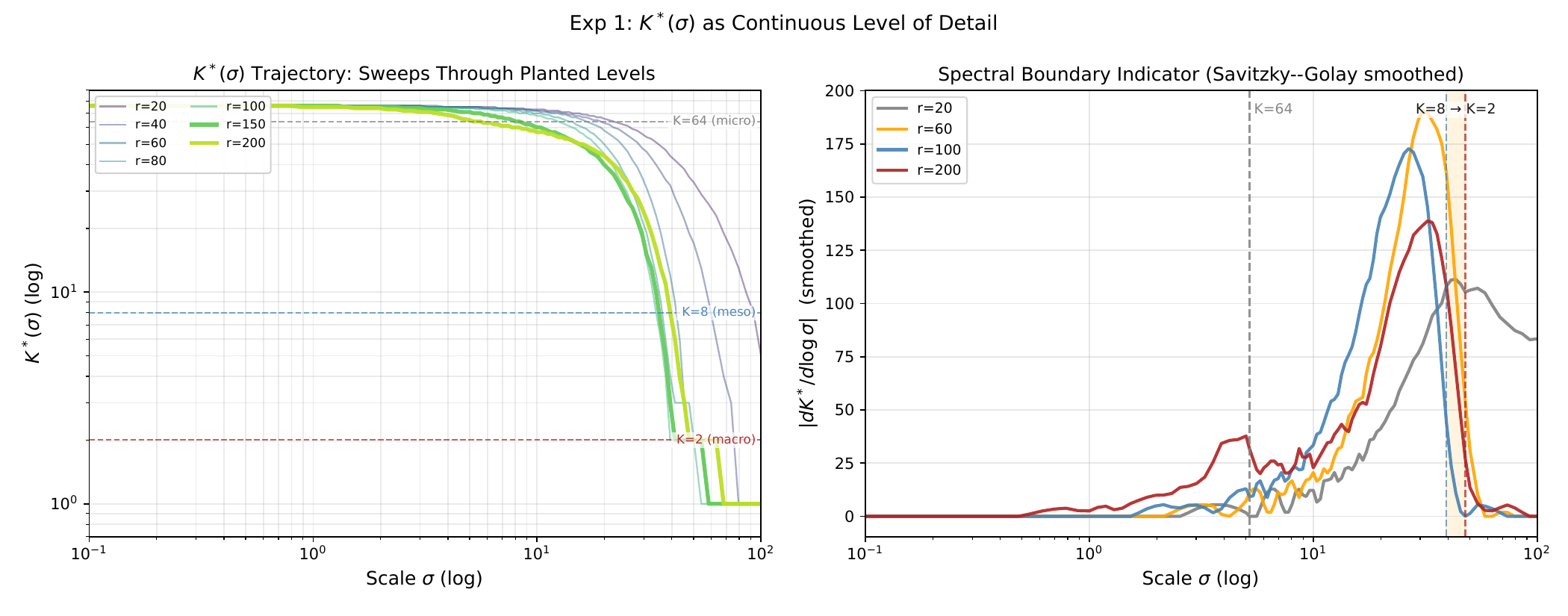}
\caption{Effective dimensionality $K^*(\sigma)$ vs.\ diffusion scale for varying signal strength $r$ (left, log--log axes; planted levels $K{=}2,8,64$ shown as horizontal dashed lines). At high $r$ the trajectory sweeps through all three planted levels in order, with $\sigma$-separations equispaced on the log axis; at low $r$ ($r{\le}40$) the trajectory stops above $K{=}8$, signalling sub-Kesten--Stigum SNR. Right: smoothed $|dK^*/d\log\sigma|$ (Savitzky--Golay window 11) as a diagnostic of effective-dimensionality changes (not the composite boundary score $S(\sigma)$ of Algorithm~\ref{alg:boundary}); concentrates at scales where $K^*$ crosses the planted levels (vertical reference lines for $r{=}200$). $K{=}8$ and $K{=}2$ are within $1.2\times$ of each other in $\sigma$ in the saturation regime, hence shown as a band.}
\label{fig:kstar_trajectory}
\end{figure}

Three key findings: \textbf{(1)}~$K^*(\sigma)$ sweeps through all three planted hierarchy levels ($K{=}2$, $K{=}8$, $K{=}64$) in order as $\sigma$ grows. \textbf{(2)}~Each level recovers at its own SNR threshold and saturates above it: macro reaches ARI~$1.00$ at $r{=}150$, meso $0.91$ at $r{=}200$, and \emph{micro $1.00$ at $r{=}200$} (Table~\ref{tab:exp1_ari}, single seed)---the finer the level, the higher the $r$ required, consistent with the Kesten--Stigum picture. Multi-seed evidence at meso (Table~\ref{tab:ablation_meso}, full Poincar\'e at $r{=}200$, 50-seed median) is consistent with this saturation. \textbf{(3)}~At $r{=}200$, the full SLoD pipeline reaches meso ARI $0.89$ [$0.86$, $0.92$] (Table~\ref{tab:ablation_meso}, 50-seed median, BCa $95\%$ CI) versus Louvain~$0.49$, greedy modularity~$0.58$, and Leiden CPM~$0.84$ (best~$\gamma$); the baseline numbers are evaluated under their canonical single-config protocols and are not directly comparable to the multi-seed median of SLoD without a matched-protocol sweep---a known asymmetry tracked as a deferred item. The substantive claim is structural rather than numerical: $K^*(\sigma)$ detects all three planted levels from a \emph{single} spectral sweep without a per-level resolution parameter, whereas Leiden CPM requires a different $\gamma$ per level.

\begin{figure}[t]
\centering
\includegraphics[width=\columnwidth]{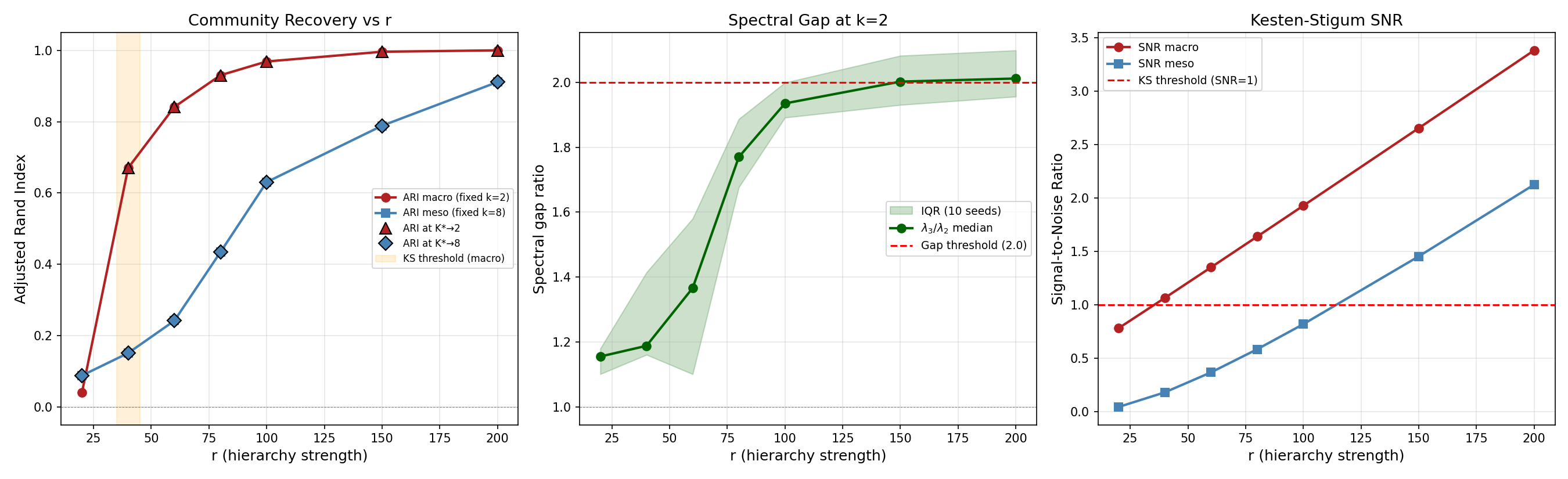}
\caption{Phase transition in boundary recovery. (a)~ARI at macro and meso scales vs.\ $r$ (single seed=42); shaded region marks Kesten--Stigum threshold. (b)~Spectral gap $\lambda_3/\lambda_2$ at $k{=}2$ on the \emph{direct combinatorial HSBM Laplacian} vs $r$ — \emph{median across 10 seeds (42--51) with shaded interquartile range}. The IQR widens at $r{=}60$--$100$, the regime where meso modes separate from the macro Fiedler value (high finite-size variance), then narrows for $r{\ge}100$ once the meso band stabilises above the macro--meso gap threshold ($\lambda_3/\lambda_2{>}2$). The earlier single-seed (42) trajectory showed a spurious dip at $r{=}80$ that is now visibly an extreme of the IQR rather than a structural feature; macro recovery (panel a) was unaffected. (c)~Theoretical signal-to-noise ratio; dashed line marks SNR$\,{=}\,1$ (Kesten--Stigum threshold; macro crosses at $r{\approx}40$, meso at $r{\approx}100$).}
\label{fig:phase_transition}
\end{figure}

\subsection{Experiment 2: Hierarchical Consistency on WordNet}\label{sec:exp2}

\paragraph{Setup.} WordNet 3.0 noun hierarchy \cite{fellbaum1998wordnet}: ${\sim}$82K synsets, DAG structure, max depth 19. Embedded in $\mathbb{B}^{10}$ via Poincar\'{e} embeddings \cite{nickel2017poincare}. BoundaryScan run from 100 leaf nodes stratified by depth (quantile binning into 10 depth bins with proportional sampling, fixed seed for reproducibility).

\paragraph{Results.} We deliberately chose WordNet because the ground-truth hierarchy is unambiguous and the embedding quality is high---both removing alternative explanations for the $\tau$ result. Embedding quality gates confirmed hierarchy preservation (HP$\,{=}\,0.994$) and sibling proximity (SP$\,{=}\,0.937$). Table~\ref{tab:exp2_metrics} summarizes BoundaryScan results. Extension to noisier or contested hierarchies (citation networks, biological taxonomies) is in Open Question~5.

\begin{table}[h]
\centering
\caption{Experiment~2 metrics on full WordNet noun hierarchy ($N{=}82{,}115$).}
\label{tab:exp2_metrics}
\small
\begin{tabular}{@{}lcc@{}}
\toprule
Metric & Value & Interpretation \\
\midrule
Kendall $\tau$ & 0.79 [0.75, 0.83] & strong ordering (bootstrap 95\% CI) \\
Recall@1 & 0.56 & 56\% of true depths detected $\pm$1 level \\
Recall@2 & 0.75 & 75\% detected $\pm$2 levels \\
Precision@0.5 & 0.79 & 79\% of detections near a true depth \\
Avg.\ peaks/node & 3.9 & $\sim$3--4 scale boundaries per leaf \\
\bottomrule
\end{tabular}
\end{table}

\begin{figure}[t]
\centering
\includegraphics[width=\columnwidth]{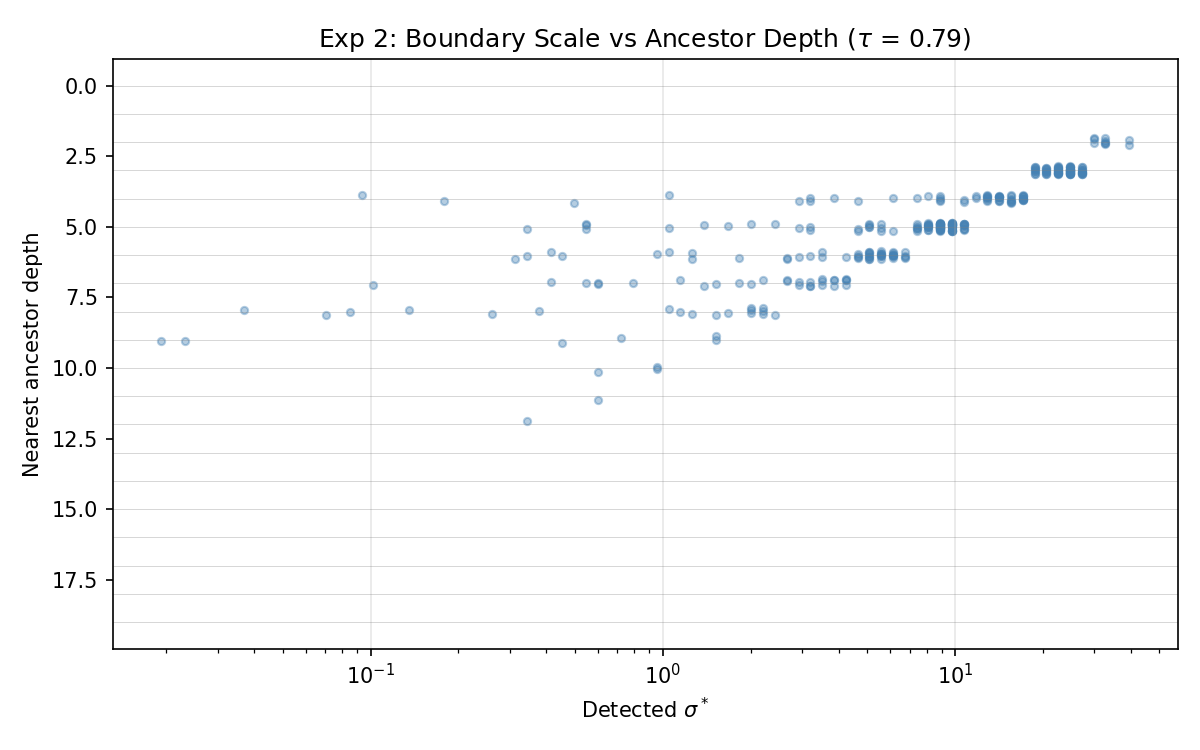}
\caption{Detected boundary scale $\sigma^*$ vs.\ true ancestor depth for 100 leaf nodes in WordNet ($N{=}82{,}115$). Kendall $\tau = 0.79$ confirms larger diffusion scales correspond to shallower ancestors.}
\label{fig:wordnet_scatter}
\end{figure}

Key findings: \textbf{(1)}~Kendall's $\tau{=}0.79$ validates that BoundaryScan correctly orders boundaries on a real-world DAG with 82K nodes. \textbf{(2)}~Precision@0.5$\,{=}\,0.79$ indicates most detected boundaries correspond to genuine ancestor depths; BoundaryScan is conservative (finds real boundaries, misses subtle deep ones). \textbf{(3)}~The Fr\'{e}chet mean at $\sigma^*$ acts as a community centroid rather than a pointer to a specific ancestor---desirable for graph-based retrieval where the appropriate abstraction level matters more than a specific node.

\subsection{Experiment 3: Ablation Study}\label{sec:ablation}

To disentangle the contribution of each pipeline component, we sweep six design dimensions on HSBM across inter-level ratios $r \in \{20, 40, 60, 80, 100, 150, 200\}$ with the 2{$\times$}4{$\times$}8 planted hierarchy and 50 Poincar\'e-MDS seeds per $r$ ($42$--$91$). Varying $r$ moves the task from the near-Kesten--Stigum detection regime ($r{\le}40$) through an intermediate regime into the saturated regime ($r{\ge}150$). The ablation dimensions are: (i) composite-score weights $(\alpha_1,\alpha_2,\alpha_3)$ combining representation velocity $V$, weight divergence $D_w$, and neighborhood churn $C_k$; (ii) MAD peak-picker multiplier ($\beta\in\{1,2,3\}$, distinct from the composite weights $\alpha_1,\alpha_2,\alpha_3$); (iii) Laplacian source (kNN-of-Poincar\'e-embedding vs.\ direct combinatorial Laplacian); (iv) embedding geometry (Poincar\'e vs.\ Euclidean Sammon MDS, with matched stress functional, optimiser family, iteration budget, dimension, and initialisation scale, differing only in the manifold metric and its retraction; Sammon-MDS is chosen for matched-baseline comparability, while production deployments would typically use learned embeddings such as Nickel--Kiela~\cite{nickel2017poincare}); (v) kNN edge weighting (Gaussian vs.\ binary); and (vi) a \emph{random-points} control that keeps the Poincar\'e-kNN Laplacian but replaces the points with i.i.d.\ Gaussian noise. For each configuration we run BoundaryScan from focus node $0$ on a common log-spaced scale grid, pick the top-scoring peak, and compute macro ($K^*{=}2$) and meso ($K^*{=}8$) ARI via fixed-$K$ spectral clustering at that $\sigma^*$ (KMeans \texttt{random\_state}${=}42$, $n_{\mathrm{init}}{=}10$). Bootstrap CIs are BCa with $n{=}10{,}000$ resamples; reported coverage is per-cell and \emph{uncorrected} for multiple comparisons across the $8\times7$ grid.

\input{tables/ablation.tex}

Six takeaways, grounded in 50-seed bootstrap CIs of Tables~\ref{tab:ablation_macro}--\ref{tab:ablation_meso} and the score-profile Figure~\ref{fig:ablation_profiles}.

\textbf{(a)~Laplacian construction is the load-bearing geometric step.} The direct-Laplacian row (kept Poincar\'e coordinates, swapped Laplacian for the combinatorial HSBM Laplacian) fails macro recovery completely at $r{=}20$ (median ARI ${=}0.00$ with $[0,0]$ CI) and is bimodal at $r{=}40$ (median ${=}0.00$ with the upper CI bound at $0.68$, indicating $\sim$half the seeds recover and $\sim$half fail), while the default kNN-of-embedding path achieves macro ARI ${=}0.42$ at $r{=}20$ and $0.72$ at $r{=}40$ on the same seeds---isolating the kNN graph construction as the geometric component most sensitive to SNR. At meso, the direct-Laplacian row trails the default by a consistent $9$--$18$\,pp at $r{\ge}60$ (Table~\ref{tab:ablation_meso}), so the kNN advantage is not just a low-SNR artefact. This is not merely an "embedding-helps-over-no-embedding" preprocessing effect: the Euclidean-MDS row (takeaway~(b)) also uses an embedding step but loses meso ARI to the Poincar\'e default---the gain therefore tracks hyperbolic geometry specifically, not preprocessing in general. Inversely, the random-points row (kept the Poincar\'e-kNN Laplacian, replaced coordinates with i.i.d.\ Gaussian noise) tracks the default within $\leq 5$\,pp at every macro and meso $r$. We frame this finding precisely as a two-step claim. \emph{(1, empirical)} The composite score $S(\sigma) = \alpha_1 \hat{V} + \alpha_2 \hat{D} + \alpha_3 \hat{C}$ uses two point-cloud-dependent indicators ($V$ via the Fr\'echet centroid $\Phi_\sigma$, and $C_k$ via $k$-NN of the points) and one point-cloud-independent indicator ($D_w$, depending only on $L$ and $\sigma$); the empirical observation is that the $\sigma^*$ chosen by this composite is similar between full Poincar\'e and random-points configurations, i.e.\ the $V$- and $C_k$-driven contributions to the peak location are robust to substituting i.i.d.\ Gaussian noise for the points. \emph{(2, structural)} Given the same $\sigma^*$, the downstream ARI is fixed-$K$ spectral clustering on the Laplacian eigenvectors at that scale, and those eigenvectors are by construction a function of $L$ alone (not of the point cloud), so the ARI is identical between configurations by construction. The within-$5$\,pp tracking of the random-points row therefore shows that \emph{$\sigma^*$ selection is robust to the point cloud used for the velocity and churn indicators}---not that the downstream Fr\'echet-mean centroids are themselves equivalent under random aggregation. A direct test of Fr\'echet-coordinate equivalence (the trajectory distance $\|\Phi_\sigma^{\mathrm{default}}-\Phi_\sigma^{\mathrm{random}}\|_{\HH}$) is deferred to the extended version. The practical takeaway: at deployment the Laplacian (or its top-$K_{\mathrm{eigs}}$ eigenpairs) is the object that needs to be shipped, and the choice of aggregation point cloud is a separate downstream question with its own coherence guarantee (Theorem~\ref{thm:coherence}).

\textbf{(b)~Embedding geometry does matter, but through the kNN graph.} Replacing Poincar\'e MDS with a Euclidean Sammon MDS (matched stress functional, Adam-family optimiser with the metric's natural retraction, identical iteration budget, dimension, and initialisation scale) yields a Laplacian with a consistently shorter useful scale range ($\sigma^*\!\approx\!0.4$ vs.\ $\approx\!7$, Figure~\ref{fig:ablation_profiles}(b)) and a small but statistically robust meso-ARI loss at $r{\ge}80$. Aggregating the 200 paired $(r, \text{seed})$ pairs at $r{\in}\{80, 100, 150, 200\}$ into a single across-cell paired effect avoids the multiplicity issue inherent in per-cell non-overlap claims: the median Poincar\'e-minus-Euclidean meso-ARI gap is $5.7$\,pp with BCa $95\%$ CI $[4.8, 6.7]$\,pp ($n_\mathrm{pairs}{=}200$, Wilcoxon signed-rank $p < 10^{-15}$). The per-$r$ cell medians decompose as $-6$\,pp at $r{=}80$, $-4$\,pp at $r{=}100$, $-10$\,pp at $r{=}150$, $-8$\,pp at $r{=}200$ (Table~\ref{tab:ablation_meso}). At macro the gap is within $\leq 3$\,pp at every $r$. The loss concentrates at the meso scale and at higher SNR, where the more accurate hyperbolic kNN topology starts to matter for finer partitions, and is attributable to the change in graph construction, not to the Fr\'echet step (per (a)).

\textbf{(c)~Binary kNN matches Gaussian-weighted kNN across the full range.} Stripping the Gaussian edge weight (retaining only the symmetric kNN topology) changes macro and meso ARI by $\leq 3$\,pp at every $r$, with overlapping bootstrap CIs in every cell of Tables~\ref{tab:ablation_macro}--\ref{tab:ablation_meso}. The Gaussian bandwidth $h_{\mathrm{kNN}}$ (median pairwise distance in the kNN graph; distinct from the Kendall $\tau$ used as the WordNet headline metric in \S\ref{sec:exp2}) is therefore largely cosmetic for both macro and meso recovery across the SNR range we tested, and a practitioner can drop $h_{\mathrm{kNN}}$ in favour of a binary kNN topology without measurable loss on this hierarchy.

\textbf{(d)~Single-indicator roles depend on the geometry.} Under Poincar\'e all three single indicators stay within $5$\,pp of the full composite at $r{=}200$ ($V$-only $0.88$, $C_k$-only $0.86$, $D_w$-only $0.84$ vs full $0.89$), so the composite mostly buys robustness rather than headline ARI. Under Euclidean the ordering inverts and the spread widens: $D_w$-only becomes the strongest single indicator ($0.82$ meso at $r{=}200$, matching full at $0.81$), while $V$-only and $C_k$-only collapse to $0.61$ and $0.60$ respectively---a $20$\,pp drop. $D_w$ is a divergence between probability distributions and is metric-agnostic; $V$ (hyperbolic velocity of the Fr\'echet mean) and $C_k$ (neighbourhood churn in the embedding) are metric-dependent and reward hyperbolic structure. The $1/3,1/3,1/3$ default composite is consequently tuned to the Poincar\'e geometry and would benefit from geometry-specific re-tuning in a Euclidean deployment.

\textbf{(e)~The three indicators locate three distinct diffusion events.} At $r{=}200$ the median $\sigma^*$ separates cleanly across single-indicator configurations: $D_w$-only at $\sigma^* \approx 2.7$ (earliest), $V$-only at $\sigma^* \approx 5.6$ (close to the full composite's $\sigma^* \approx 6.1$), and $C_k$-only at $\sigma^* \approx 14.9$ (latest). The full composite locks onto the $V$-driven mode since its score is dominated by $V$ and $D_w$ at the relevant $\sigma$ band, while $C_k$ contributes a separate later peak that the composite passes over. The three events---weight-distribution shift, Fr\'echet displacement, neighbourhood churn---occur in different phases of the diffusion. The composite therefore does not collapse to a single dominant indicator and offers \emph{bounded downside} rather than a multiplicative gain.

\textbf{(f)~Peak-picker defaults are robust across all regimes and geometries.} The MAD multiplier $\beta \in \{1,2,3\}$ changes the number of detected peaks but not the top-scoring one at any $(r, \text{config})$ combination across our $7\,r \times 50$-seed lax/strict sweep, supporting the ``data-driven resolution discovery'' framing of \S\ref{sec:discussion}.

\begin{figure}[h]
\centering
\includegraphics[width=0.95\columnwidth]{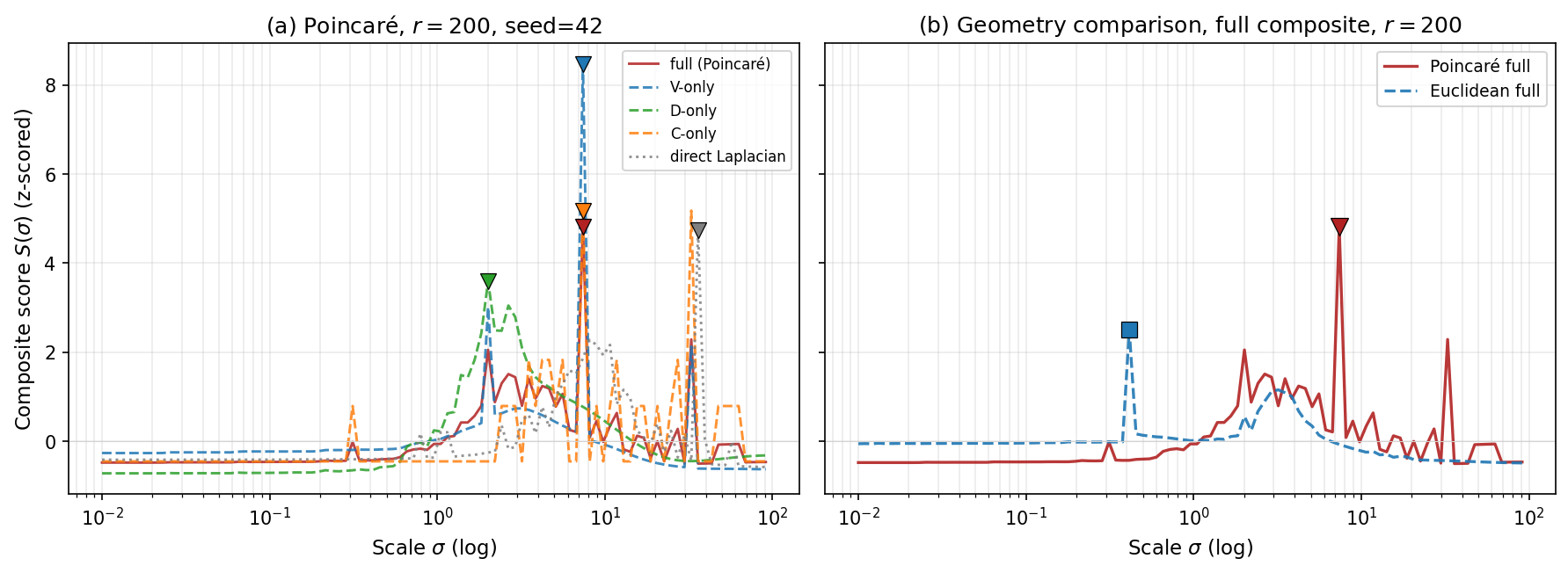}
\caption{Composite-score profiles $S(\sigma)$ at $r{=}200$, seed 42, with the top-scoring peak of each configuration marked. \textbf{(a)~Indicator ablation under Poincar\'e embedding (single seed=42):} $V$-only and the full composite share a dominant peak at $\sigma^*\!\approx\!7$, while $D_w$-only has a distinct earlier peak at $\sigma^*\!\approx\!2$. $C_k$-only at this seed coincides with the $V$-only peak; the multi-seed median analysis in takeaway~(e) shows $C_k$-only's typical peak is later ($\sigma^*\!\approx\!14.9$). The direct-Laplacian variant picks a late $\sigma^*\!\approx\!36$ on a different Laplacian spectrum. \textbf{(b)~Geometry comparison for the full composite:} the Euclidean-MDS Laplacian has a useful scale range more than an order of magnitude lower than the Poincar\'e one, consistent with takeaway~(b).}
\label{fig:ablation_profiles}
\end{figure}

\section{Related Work}\label{sec:related}

\paragraph{Hyperbolic Embeddings and Graph Neural Networks.}
The Poincar\'e ball embeds tree-structured hierarchies with $(1{+}\varepsilon)$ distortion \cite{sarkar2011low, nickel2017poincare}, enabling tree-aware computations that Euclidean spaces cannot support. Hyperbolic GNNs \cite{ganea2018hyperbolic, chami2019hyperbolic} extend this to learnable architectures that outperform Euclidean counterparts on hierarchical tasks. Our work instead introduces a non-learnable operator---heat kernel diffusion on the Poincar\'e ball---that exploits hyperbolic geometry for coherent multi-scale aggregation without requiring training.

\paragraph{Community Detection.}
Classical community detection includes modularity-based methods such as Louvain \cite{blondel2008fast} and Leiden \cite{traag2019louvain}, whose partitions depend on a user-chosen resolution and are subject to the modularity resolution limit \cite{fortunato2007resolution}, as well as information-theoretic methods such as Infomap \cite{rosvall2008maps}, which likewise produce discrete partitions but are not framed through a $\gamma$-style modularity parameter. Hierarchical stochastic block models \cite{peixoto2014hierarchical} provide model-selection over resolution but still yield discrete partitions at each level. SLoD departs from this family by defining a \emph{continuous} operator whose output at any scale $\sigma$ is a well-defined representation rather than a partition; meaningful levels emerge from spectral structure rather than from maximizing a modularity-style objective.

\paragraph{Scale-Space and Multi-Scale Graph Analysis.}
Scale-space theory \cite{lindeberg1994scale} formalized heat diffusion as a principled multi-resolution smoothing family satisfying causality axioms. Its graph analogue appears in spectral wavelets \cite{hammond2011wavelets}, diffusion maps \cite{coifman2006diffusion}, and Heat Kernel Signatures \cite{sun2009concise} for shape analysis. Closest in spirit is Markov Stability \cite{delvenne2010stability, lambiotte2014random}, which interprets community structure as a function of random-walk time---formally identical to SLoD's diffusion scale $\sigma$ on the original graph (the heat kernel $\exp(-\sigma L)$ is the continuous-time random-walk transition). SLoD differs on three measured axes. \textbf{Substrate:} Markov Stability uses the combinatorial Laplacian on the original graph; SLoD uses the kNN-of-Poincar\'e-embedding Laplacian, which empirically improves meso recovery by $9$--$18$\,pp at $r{\ge}60$ (Table~\ref{tab:ablation_meso}, direct-Laplacian row). \textbf{Scale selection:} Markov Stability optimises a stability functional over partitions; SLoD's BoundaryScan composes a focus-driven indicator triple $(V, D_w, C_k)$ on a continuous representation trajectory, yielding emergent boundaries from spectral gaps rather than partition optimisation. \textbf{Output:} Markov Stability produces discrete partitions per scale; SLoD produces a continuous Fr\'echet-centroid trajectory $\sigma \mapsto \Phi_\sigma$ with the option of fixed-$K$ spectral clustering at detected boundaries. Composing this scale-space philosophy with hyperbolic geometry's low-distortion tree embedding and with the emergent boundary mechanism is the combination not previously explored in graph-based AI systems.

\paragraph{Graph-Based Retrieval and Agent Memory.}
Graph-based retrieval systems such as GraphRAG \cite{edge2024graphrag} organize documents into hierarchical community summaries for multi-level querying; persistent agent memory systems \cite{packer2023memgpt} and recent surveys \cite{zhang2024memory} inherit this pattern. All inherit the discrete-partition paradigm of community detection, committing to a fixed resolution at indexing time. Hybrid retrieval approaches such as HyDE \cite{gao2023hyde} operate at the token/embedding level and do not exploit graph structure. SLoD provides a complementary layer: a continuous resolution operator upstream of either graph or token retrievers, enabling smooth traversal of the abstraction hierarchy at query time with formal coherence guarantees (Theorem~\ref{thm:coherence}).

\section{Discussion and Conclusion}\label{sec:discussion}

SLoD addresses a gap in graph-based AI systems: the absence of principled, continuous resolution control. We conclude by highlighting implications for the graphs-across-AI community and open questions.

\paragraph{Data-Driven Resolution Discovery.}
In the tested hierarchical regimes (HSBM, WordNet), SLoD removes the need for a per-level Leiden $\gamma$ sweep. Where Leiden requires sweeping $\gamma$ (and suffers from Louvain's resolution limit \cite{fortunato2007resolution}), the spectral boundary mechanism recovers multiple meaningful scales from a single sweep. An agent can query at any $\sigma$ and receive a representation at the corresponding abstraction level---without resolution-parameter tuning. Internal aggregation weights ($\alpha_1, \alpha_2, \alpha_3$) and the MAD peak-picking threshold use defaults (Algorithm~\ref{alg:boundary}) validated on HSBM and WordNet; their transferability to novel domains remains to be established.

\paragraph{Graph-Native Abstraction for Retrieval.}
Current GraphRAG pipelines use discrete community detection for multi-level summarization. SLoD offers continuous zoom: retrieval at the optimal abstraction for each query. The coherence guarantee (Theorem~\ref{thm:coherence}) ensures that nearby scales produce semantically related representations, enabling smooth traversal of the abstraction hierarchy during retrieval.

\paragraph{Deployment in Agent Memory Systems.}
We outline how SLoD plugs into a graph-based agent memory pipeline.
\textbf{(1)}~The agent accumulates experience into a knowledge graph $G$;
\textbf{(2)}~an offline pass builds a Poincar\'e embedding and runs BoundaryScan (Algorithm~\ref{alg:boundary}) to produce an \emph{abstraction index} comprising boundary scales $\{\sigma^*_i\}$, effective dimensionalities $\{K^*_i\}$, and per-scale Fr\'echet centroids;
\textbf{(3)}~at query time, the agent maps a query $q$ to a target scale $\sigma_q$ and retrieves from the nearest $\sigma^* \in \{\sigma^*_i\}$ via the centroid or the $K^*$-cluster at that scale;
\textbf{(4)}~usage updates edge weights (e.g., via Hebbian co-activation), which trigger periodic re-indexing.
SLoD is thus a natural drop-in for the hierarchical indexing layer of graph-based retrieval systems such as GraphRAG \cite{edge2024graphrag}, where Leiden's discrete community partitions could be replaced by continuous SLoD boundaries. How the agent picks $\sigma_q$ is itself a research question: from an information-theoretic / resource-budget perspective \cite{tishby2000information, friston2010free}, $\sigma_q$ should scale with query complexity under a cognitive budget---broad contextual queries favour larger $\sigma$ (more aggregation), precise-recall queries favour smaller $\sigma$ (less smoothing). A principled derivation of this mapping is beyond the scope of this discussion paper and is flagged as a key open problem for deployment.

\paragraph{Where the Hyperbolic Geometry Enters.}
Experiment 3 localises the contribution of the Poincar\'e substrate. The Laplacian whose spectrum drives BoundaryScan is built from hyperbolic-distance kNN edges; replacing those with Euclidean-distance kNN loses meso ARI by an $r$-dependent margin in the $4$--$10$\,pp range at $r{\ge}80$, peaking at $-10$\,pp at $r{=}150$ (Table~\ref{tab:ablation_meso}; the across-cell paired effect is $5.7$\,pp with BCa $95\%$ CI $[4.8, 6.7]$\,pp, see \S\ref{sec:ablation} takeaway~(b)). In contrast, the Poincar\'e \emph{coordinates} used for the Fr\'echet mean are largely interchangeable with i.i.d.\ Gaussian noise at the same operating points (takeaway~(a) of \S\ref{sec:ablation}). This is a sharper positioning rather than a weaker one: the value of the hyperbolic embedding is concentrated in the single upstream step of neighbourhood definition, while the downstream Fr\'echet mean remains the principled Poincar\'e-ball summary object that retrieval systems query at the selected scale. Theorem~\ref{thm:coherence} is a statement about the \emph{trajectory} $\sigma \mapsto \Phi_\sigma$ of the Fr\'echet mean; no analogous coherence guarantee holds under random aggregation coordinates, even in regimes where their single-scale ARI matches. For deployment, this means the Laplacian (or its top-$K_{\mathrm{eigs}}$ eigenpairs) is the object that needs to be shipped and refreshed; the raw point cloud is not a query-time dependency at most operating points.

\paragraph{Spectral Gaps as Structural Fingerprints.}
The detected scale boundaries encode intrinsic graph structure independent of any particular embedding. This has implications for graph neural architectures: spectral boundary positions could inform multi-scale pooling strategies, replacing hand-designed hierarchical pooling with data-driven scale selection.

\paragraph{Robustness Across Density Regimes.}
We observe that Sammon embedding quality $Q_S$ decreases monotonically as HSBM density grows (from $Q_S{=}0.783$ at $r{=}20$ to $Q_S{=}0.726$ at $r{=}200$, a $7.3\%$ drop), consistent with the general fact that denser graphs have shorter diameters and more cycles and are therefore harder to embed with low distortion into a tree-like hyperbolic space~\cite{sarkar2011low}. Yet the direct-Laplacian variant of BoundaryScan achieves macro ARI${=}1.00$ at $r{=}200$ (Table~\ref{tab:exp1_ari}): macro recovery does not require a faithful embedding. At meso scale the direct variant costs 21\,pp at $r{=}200$ (Table~\ref{tab:ablation_meso}), so the embedding still contributes at finer resolution. This refines the broader claim: SLoD's macro-scale mechanism decouples from embedding fidelity, while meso-scale recovery still benefits from it---a property useful for dense real-world knowledge graphs where even a lossy embedding is better than none for fine-grained retrieval.

\paragraph{Formal Guarantees for Trustworthy Retrieval.}
The bounded approximation error $O(\sigma)$ and $(1{+}\varepsilon)$ coherence bound provide guarantees that heuristic community detection methods lack. For applications requiring trustworthy retrieval---medical knowledge, legal reasoning, safety-critical systems---such guarantees are essential.

\paragraph{Scalability and Approximation.}
Our implementation uses Lanczos partial eigendecomposition (\S\ref{sec:method}), costing $O(N K_{\mathrm{eigs}}^2)$ time and $O(N K_{\mathrm{eigs}})$ memory with $K_{\mathrm{eigs}}$ retained eigenvalues. We empirically validate this scaling on HSBM graphs across $N \in [978, 15{,}771]$ nodes (post--largest-connected-component filter, from requested sizes $\{1024, 2048, 4096, 8192, 16384\}$) with a fixed hierarchy ($2{\times}4{\times}8$ micro-communities) and $K_{\mathrm{eigs}}{=}50$: BoundaryScan wall-clock scales as $N^{0.77}$ and peak Python allocation as $N^{0.98}$ in log--log fits (Fig.~\ref{fig:scaling}), consistent with the linear factor in $N$ for fixed $K_{\mathrm{eigs}}$. This was practical on WordNet ($N{\approx}82{,}000$) at $K_{\mathrm{eigs}}{=}50$, but the approach is not intrinsic to the framework. For larger graphs, three standard approximation paths apply: (i) Chebyshev polynomial expansion of $\exp(-\sigma L)$ as used for graph wavelets \cite{hammond2011wavelets}, reducing cost to $O(m d_{\mathrm{cheb}})$ where $m$ is the edge count and $d_{\mathrm{cheb}}$ is the polynomial degree (distinct from $K_{\mathrm{eigs}}$); (ii) Nystr\"om-style spectral sampling from a subset of nodes; (iii) localized random-walk approximations that compute a focus-specific heat kernel without materializing the global spectral decomposition. A systematic evaluation at Wikidata scale ($N{>}10^7$) is left to future work. All timings reported here were measured single-threaded on an Apple M-series CPU with 16\,GB RAM; Lanczos eigendecomposition is the dominant cost at our $K_{\mathrm{eigs}}$.

\begin{figure}[h]
\centering
\includegraphics[width=0.95\columnwidth]{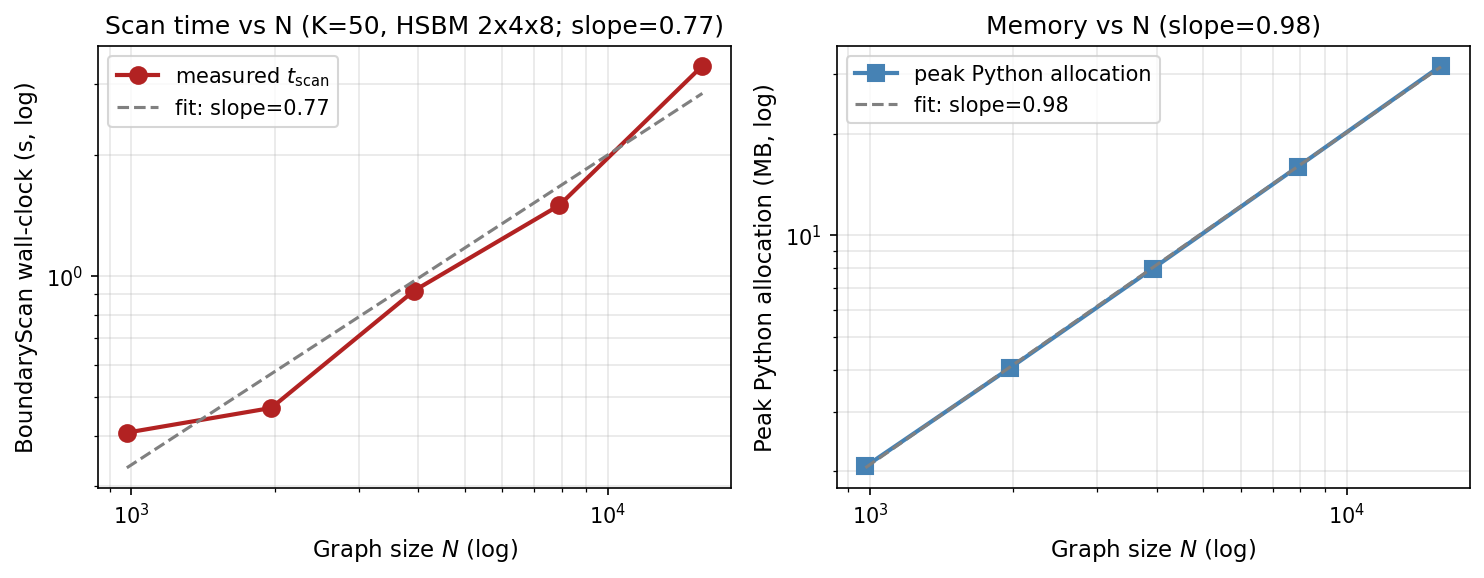}
\caption{BoundaryScan scaling on HSBM with a fixed $2{\times}4{\times}8$ hierarchy and $K_{\mathrm{eigs}}{=}50$. \textbf{Left:} wall-clock scan time, log--log slope $0.77$ (the sub-linear trend at small $N$ is constant-overhead-dominated). \textbf{Right:} peak Python allocation (measured via \texttt{tracemalloc}), log--log slope $0.98$. Both consistent with $O(N K_{\mathrm{eigs}}^2)$ time and $O(N K_{\mathrm{eigs}})$ memory (Lanczos partial eigendecomposition) for fixed $K_{\mathrm{eigs}}$.}
\label{fig:scaling}
\end{figure}

\paragraph{Limitations --- Non-Hierarchical Structures.}
SLoD's coherence guarantee (Theorem~\ref{thm:coherence}) depends on a low-distortion Sarkar embedding of a tree-like graph. We frame Theorem~\ref{thm:coherence} as a result that establishes the coherence property in an idealized setting (exact tree, Sarkar embedding); the empirical results on HSBM (noisy hierarchy) and WordNet (DAG with multiple inheritance, Nickel--Kiela embedding) provide evidence that the property holds approximately under the noise present in these realistic regimes. On graphs without hierarchical structure, three regimes emerge. \textbf{(a)}~\emph{Flat graphs} with a single community scale produce a degenerate BoundaryScan output---one boundary at $K^*{=}|\text{communities}|$, with no further hierarchy. This is the expected behaviour and is not a failure mode. \textbf{(b)}~\emph{Overlapping communities} violate the disjoint-partition assumption underlying the graph Laplacian spectrum, and may yield spurious boundaries; the $(1{+}\varepsilon)$ coherence bound degrades as Sarkar distortion grows. \textbf{(c)}~\emph{Random graphs} (Erd\H{o}s--R\'enyi) lack spectral gaps entirely; the MAD-threshold peak picker should return an empty boundary set, though the fallback "highest local maximum" rule can produce a single false positive. Formal characterization of these degradation modes---and automatic detection of the applicable regime from spectral signatures---remains open.

\paragraph{Scope --- Explicit vs.\ Latent Hierarchy.}
The validation regime here is the \emph{explicit-hierarchy} setting: HSBM with planted levels; WordNet with curated taxonomic structure and a high-quality Poincar\'e embedding (HP$\,{=}\,0.994$, SP$\,{=}\,0.937$). Many real-world knowledge graphs (corporate KGs, GraphRAG entity graphs) instead carry hierarchy only \emph{implicitly}---structure that must be inferred rather than read off. Whether SLoD's spectral boundary mechanism can surface such latent hierarchy when it is present is an open methodological question, distinct from the (a)/(b)/(c) regimes above which describe behaviour when no meaningful hierarchy is present at all. We expect the answer to depend on whether the upstream embedding step preserves enough hierarchical signal for kNN-graph spectral gaps to encode it; ablation~(b) of \S\ref{sec:ablation} (Euclidean MDS loses meso ARI) is an early hint that embedding choice is consequential here.

\paragraph{Limitations --- Perturbation Robustness.}
Beyond structural assumptions, several perturbation dimensions merit systematic study. \textbf{(a)}~\emph{Edge noise} (random add/remove): Theorem~\ref{thm:coherence} provides a partial bound via Sarkar distortion $\varepsilon$, but empirical characterization is open. \textbf{(b)}~\emph{Embedding noise}: the bootstrap CI for Kendall $\tau$ on WordNet (Table~\ref{tab:exp2_metrics}, $[0.75, 0.83]$) captures focus-node variance at one embedding seed; cross-seed variance of the embedding itself is unmeasured. \textbf{(c)}~\emph{Clustering variance}: the §6.3 bootstrap CIs are computed over graph-generation and MDS-optimisation seeds with $k$-means initialisation pinned (\texttt{random\_state}${=}42$, $n_{\mathrm{init}}{=}10$); cells where two near-degenerate eigenmodes compete for the partition (notably meso $r{\in}\{80,100\}$) may therefore have CI widths that under-report true sampling variance by a factor of $\sim$1.3--2$\times$. The reported intervals should be read as lower bounds on the full sampling distribution. \textbf{(d)}~\emph{Hyperparameter sensitivity}: defaults for $(\alpha_1, \alpha_2, \alpha_3, R, \text{MAD multiplier})$ were set on HSBM and WordNet; a systematic sensitivity table is deferred to the extended version. \textbf{(e)}~\emph{Node subsampling}: relevant for approximate SLoD on large graphs via Nystr\"om-style methods (see Scalability). \textbf{(f)}~\emph{Adversarial perturbations}: minimal edge edits that shift detected boundaries, relevant to deployment security. Dimensions (e) and (f) remain open; (d) is partially addressed by the defaults' stability across two distinct domains (synthetic HSBM, real WordNet).

\paragraph{Open Questions.}
Several directions merit investigation.
\textbf{(1)}~How does SLoD interact with \emph{dynamic graphs} where edges evolve through usage? Coupling with Hebbian co-activation learning and incremental spectral updates is a promising path.
\textbf{(2)}~Can spectral boundary detection guide \emph{GNN architecture design}---e.g., informing layer depth or pooling ratios?
\textbf{(3)}~How do boundary positions shift under \emph{adversarial edge perturbations}?
\textbf{(4)}~Extension to dense knowledge graphs, where the tree assumption no longer holds and coherence bounds degrade.
\textbf{(5)}~\emph{Broader empirical evaluation.} Current validation covers HSBM (synthetic hierarchy) and WordNet (taxonomic DAG). Natural next targets span different structural regimes: citation networks (ogbn-arxiv), biological networks (PPI, Gene Ontology), LFR benchmarks for overlapping communities, and commercial taxonomies (Amazon categories, DBpedia). Each probes an aspect of scale discovery that current experiments do not address.
\textbf{(6)}~\emph{Benchmarking against modern retrieval systems.} Retrieval-augmented systems such as GraphRAG \cite{edge2024graphrag} and HyDE \cite{gao2023hyde} operate at a different abstraction than spectral methods. A head-to-head evaluation on QA benchmarks (LoCoMo, NaturalQuestions, MS\,MARCO) would clarify when SLoD-driven hierarchical retrieval complements vs.\ competes with LLM-based methods. This requires building the full retrieval pipeline on top of SLoD boundaries and is deferred to a systems paper.
\textbf{(7)}~\emph{Direct Fr\'echet-coordinate equivalence test.} The random-points control of \S\ref{sec:ablation} establishes that $\sigma^*$ selection is robust to the aggregation point cloud, but ARI is a clustering metric on Laplacian eigenvectors and does not directly probe Fr\'echet-mean similarity. The natural follow-up is the trajectory distance $\|\Phi_\sigma^{\mathrm{default}} - \Phi_\sigma^{\mathrm{random}}\|_{\HH}$ as a function of $\sigma$, which would quantify when (and how strongly) the Poincar\'e coordinate cloud is load-bearing for the centroid summary at each scale.
\textbf{(8)}~\emph{Multi-focus aggregation.} All ablation rows use focus node $0$. A stratified sweep across multiple foci would probe whether the reported $\sigma^*$ medians are focus-specific or graph-level signatures.

\medskip
\noindent In summary, Semantic Level of Detail provides a mathematically grounded framework for discovering where meaningful abstraction boundaries lie in knowledge graphs. By letting diffusion on the graph Laplacian find the boundaries automatically, SLoD offers graph-based AI systems a principled alternative to manual resolution tuning.

\section*{Acknowledgments and Disclosures}

The author thanks the two anonymous GRAAI~2026 reviewers for constructive comments that materially improved \S\ref{sec:method} (hyperbolic positioning), \S\ref{sec:ablation} (statistical methodology), and \S\ref{sec:related} (Markov Stability differentiation). The author is a founder of Mnemoverse.AI; this work is part of the company's research program on graph-based AI memory systems. The manuscript was drafted with assistance from large language models; all methodology, experiments, claims, and responsibility for content are the author's. Source code and reproducibility scripts: \url{https://github.com/mnemoverse/mnemoverse-slod-paper} (archived at \url{https://doi.org/10.5281/zenodo.19933482}).

\appendix

\section{Proof of Theorem~\ref{thm:coherence}}\label{app:coherence}

\noindent\emph{Note on constants.} The constants in Lemmas~A.1 and~A.2 below are stated explicitly: $L_F = D/2$ for the Fr\'echet step (sharp; valid on every Hadamard manifold, attained at $N{=}2$), and $L_w = 2\,\|L\|_{1\to 1}$ for the heat-kernel weight step (where $\|L\|_{1\to 1} = \max_j \sum_i |L_{ij}|$ is the induced $\ell^1\to\ell^1$ operator norm of the graph Laplacian). An alternative $L_w = \lambda_{\max}$ form (used informally in some wavelet treatments \cite{hammond2011wavelets} for regular graphs) does not generalise to graphs with non-trivial degree variance: e.g.\ on a star graph $K_{1,n}$, the actual $L_w$ scales as $\sqrt{n}$ at the hub focus while $\lambda_{\max} \le 2$. The $\|L\|_{1\to 1}$ form used here captures the kNN-graph-conditioning dependence and reduces to a small constant for bounded degree ratio.

We restate the theorem for convenience.

\smallskip
\noindent\textbf{Theorem~\ref{thm:coherence}.} \emph{Let $\mathcal{T}$ be a weighted tree of bounded maximum degree $d_T$, embedded in $\BB$ via Sarkar embedding with distortion $\delta = (1{+}\varepsilon)$, and let $L$ be the symmetric normalized Laplacian of the kNN graph on the embedded points (with bounded kNN degree ratio so that $\|L\|_{1\to 1}$ is a small constant). For any node $v$ with subtree-edge diameter $D_{\mathcal{T}}(v)$ and $\sigma_1 < \sigma_2$,}
\[
  \dH\!\left(\Phi_{\sigma_1}(\mathcal{V}_v, v),\ \Phi_{\sigma_2}(\mathcal{V}_v, v)\right)
  \;\le\; C \cdot |\sigma_2 - \sigma_1| \cdot (1{+}\varepsilon),
\]
\emph{where $\mathcal{V}_v$ is the set of descendants of $v$ and $C = \ell\,\|L\|_{1\to 1}\,D_{\mathcal{T}}(v)$, with $\ell = \ell(d_T,\varepsilon)$ the Sarkar base edge length and $\|L\|_{1\to 1} := \max_j \sum_i|L_{ij}|$ the maximum absolute column sum of $L$.}

The proof proceeds in two Lipschitz steps: (i)~the Fr\'echet mean $\Phi$ is $D/2$-Lipschitz in its weights $w$ with respect to the $\ell^1$ norm on the simplex, sharply on every Hadamard manifold (Lemma~A.1); (ii)~the heat-kernel weights $w(\sigma)$ are $2\,\|L\|_{1\to 1}$-Lipschitz in $\sigma$ on the simplex, by direct $\ell^1$ analysis of the heat semigroup (Lemma~A.2).

\subsection*{Lemma A.1 (Lipschitz dependence of the Fr\'echet mean on weights)}

Let $\mathcal{V} = \{v_i\}_{i=1}^N \subset \BB$ be a finite set with $\mathrm{diam}(\mathcal{V}) \le D$. For weights $w, w'$ on the simplex $\Delta^{N-1}$, the Fr\'echet means
\[
  \Phi(w) := \argmin_{y \in \BB}\, \sum_{i=1}^{N} w_i\, \dH^2(y, v_i)
\]
satisfy
\begin{equation}\label{eq:appA-frechet-lipschitz}
  \dH\!\left(\Phi(w),\ \Phi(w')\right) \;\le\; L_F \cdot \|w - w'\|_1,
\end{equation}
where $L_F = D / 2$ for the Poincar\'e ball $\BB$.

\smallskip
\noindent\emph{Proof.} The proof uses only the Hadamard property (simply connected, complete, sectional curvature $\le 0$); the Poincar\'e ball with $\kappa \equiv -1$ is the relevant special case. On any Hadamard manifold, $F_w$ is strictly geodesically convex \cite{karcher1977riemannian,sturm2003probability,afsari2011riemannian} and the minimizer $\Phi(w)$ is unique. Set $\phi_i(y) := \tfrac{1}{2}\dH^2(y, v_i)$ and work with $\widetilde F_w := \sum_i w_i \phi_i = \tfrac{1}{2} F_w$, whose unique minimizer is again $\Phi(w)$.

\smallskip
\noindent\emph{Step 1 (Hessian lower bound).} For each $v_i$ and each $y \in \BB$ with $r_i := \dH(y, v_i) > 0$, the Hessian of $\phi_i$ on $T_y \BB$ has the explicit form (Karcher \cite{karcher1977riemannian}, \S6; the equivalent strong-convexity / variance form appears in Sturm \cite{sturm2003probability}, Prop.~4.4)
\[
  \mathrm{Hess}\,\phi_i(\xi, \xi) \;=\; \|\xi_\parallel\|^2 + r_i \coth(r_i)\, \|\xi_\perp\|^2,
\]
where $\xi = \xi_\parallel + \xi_\perp$ decomposes along the radial direction $\log_y(v_i)/r_i$ and its orthogonal complement (interpreted as $\xi_\perp = 0$ in the limiting case $r_i = 0$, where the radial direction is undefined and the formula reduces to $\|\xi\|^2$). Since $r \coth(r) \ge 1$ for all $r \ge 0$, $\mathrm{Hess}\,\phi_i \succeq I$ pointwise, and convex combination yields
\begin{equation}\label{eq:appA-hessian-lower}
  \mathrm{Hess}_y\,\widetilde F_w(\Phi(w)) \;=\; \sum_i w_i\,\mathrm{Hess}_y\,\phi_i(\Phi(w)) \;\succeq\; I.
\end{equation}
The lower bound is curvature-independent and tight in the radial direction; negative curvature only strengthens the tangential direction.

\smallskip
\noindent\emph{Step 2 (CAT(0) angle inequality).} For any base point $y \in \BB$ and $p, q \in \BB$,
\begin{equation}\label{eq:appA-cat0}
  \|\log_y(p) - \log_y(q)\|_{T_y \BB} \;\le\; \dH(p, q).
\end{equation}
This is the standard CAT(0) comparison: the Euclidean comparison triangle with sides $a = \dH(y,p)$, $b = \dH(y,q)$, $c = \dH(p,q)$ has angle $\bar\alpha$ at $y$ with $\cos\bar\alpha = (a^2{+}b^2{-}c^2)/(2ab)$; non-positive curvature gives manifold angle $\alpha \le \bar\alpha$ (Sturm \cite{sturm2003probability}, \S2), so by the law of cosines $\|\log_y p - \log_y q\|^2 = a^2 + b^2 - 2ab\cos\alpha \le c^2$. Applied at $y = \Phi(w)$, the tangent set $\{\log_\Phi(v_i)\}_{i=1}^N$ has tangent diameter at most $\mathrm{diam}_{\dH}(\mathcal{V}) \le D$.

\smallskip
\noindent\emph{Step 3 (implicit function theorem).} The first-order optimality condition $\nabla_y \widetilde F_w = -\sum_i w_i \log_y(v_i) = 0$ at $y = \Phi(w)$ defines $\Phi$ implicitly. Differentiating in $w \in \Delta^{N-1}$, on the simplex tangent $T \Delta^{N-1} = \{dw \in \RR^N : \sum_i dw_i = 0\}$,
\[
  \mathrm{Hess}\,\widetilde F_w(\Phi)\cdot d\Phi \;=\; \sum_i \log_\Phi(v_i)\, dw_i.
\]
Since $\sum_i dw_i = 0$, decompose $dw = dw_+ - dw_-$ with $dw_\pm \ge 0$ and $\sum_i dw_+^{(i)} = \sum_i dw_-^{(i)} = \tfrac{1}{2}\|dw\|_1$. Writing the right-hand side as $\tfrac{1}{2}\|dw\|_1\,(\bar a_+ - \bar a_-)$ with $\bar a_\pm$ the corresponding probability-weighted barycenters in $T_\Phi \BB$ of $\{\log_\Phi v_i\}$, \eqref{eq:appA-cat0} gives $\|\bar a_+ - \bar a_-\|_{T_\Phi \BB} \le D$, hence
\[
  \Big\|\sum_i \log_\Phi(v_i)\, dw_i\Big\|_{T_\Phi \BB} \;\le\; \tfrac{D}{2}\,\|dw\|_1.
\]
Combining with \eqref{eq:appA-hessian-lower}, $\|[\mathrm{Hess}\,\widetilde F_w(\Phi)]^{-1}\|_{op} \le 1$, so
\begin{equation}\label{eq:appA-pointwise}
  \|d\Phi\|_{T_\Phi \BB} \;\le\; \tfrac{D}{2}\,\|dw\|_1.
\end{equation}

\smallskip
\noindent\emph{Step 4 (path integration).} Joining $w$ to $w'$ by the segment $w(t) := (1{-}t)w + t w'$ in $\Delta^{N-1}$, the curve $t \mapsto \Phi(w(t))$ has Riemannian length bounded by integrating \eqref{eq:appA-pointwise}, and Riemannian distance is bounded by length:
\[
  \dH\!\left(\Phi(w), \Phi(w')\right) \;\le\; \int_0^1 \big\|\tfrac{d\Phi}{dt}\big\|_{T_{\Phi(t)} \BB}\, dt \;\le\; \tfrac{D}{2}\,\|w-w'\|_1,
\]
which is \eqref{eq:appA-frechet-lipschitz} with $L_F = D/2$.

\smallskip
\noindent\emph{Sharpness and curvature independence.} The constant $D/2$ is attained at $N = 2$ on every Hadamard manifold, including $\BB$. For $v_1, v_2$ at hyperbolic distance $D$, $\Phi(w_1, w_2)$ lies on the geodesic $\gamma$ from $v_1$ to $v_2$ at arclength parameter $s^*(w) = w_2 D$, by minimization of the (Euclidean, in the arclength variable) parabola $w_1 s^2 + w_2 (D{-}s)^2$. Hence $\dH(\Phi(w), \Phi(w')) = D|w_2 - w_2'|$ while $\|w-w'\|_1 = 2|w_2 - w_2'|$, giving Lipschitz ratio exactly $D/2$. The construction is independent of $\kappa$: only that $v_1, v_2$ lie on a unique geodesic of length $D$ enters. \hfill$\square$

\subsection*{Lemma A.2 (Lipschitz dependence of heat-kernel weights on $\sigma$)}

Let $L = I - D^{-1/2} W D^{-1/2}$ be the symmetric normalized Laplacian of the (Gaussian or binary) kNN graph on $\mathcal{V}$, and let
\[
  \|L\|_{1\to 1} \;:=\; \max_{j} \sum_{i} |L_{ij}|
\]
denote its induced $\ell^1\to\ell^1$ operator norm (the maximum absolute column sum). The heat-kernel weights $w(\sigma) = u(\sigma)/\|u(\sigma)\|_1$ with $u(\sigma) = e^{-\sigma L}\,\mathbf{e}_{x_0}$ from focus $x_0$ (Definition~\ref{def:weights}) satisfy
\begin{equation}\label{eq:appA-weights-lipschitz}
  \| w(\sigma_2) - w(\sigma_1) \|_1 \;\le\; 2\,\|L\|_{1\to 1}\,|\sigma_2 - \sigma_1|.
\end{equation}
For OR-symmetrized kNN graphs with degree ratio $d_{\max}/d_{\min} \le \rho_d$ (distinct from the spectral gap ratio $\rho_k$ of Proposition~\ref{prop:spectral}), $\|L\|_{1\to 1} \le 1 + \sqrt{\rho_d}$, so $L_w \le 2(1+\sqrt{\rho_d})$; in particular $L_w \le 2(1+\sqrt{2})$ for kNN graphs whose degrees lie in $[k, 2k]$ (the SLoD operating regime, \S\ref{sec:method}).

\smallskip
\noindent\emph{Proof.} Since $A := D^{-1/2} W D^{-1/2}$ is non-negative entrywise, $e^{-\sigma L} = e^{-\sigma}\sum_{m\ge 0}(\sigma A)^m/m! \ge 0$ entrywise, hence $u(\sigma) \ge 0$ on a connected graph. Set $Z(\sigma) := \mathbf{1}^T u(\sigma) = \|u(\sigma)\|_1 > 0$.

\smallskip
\noindent\emph{Step 1 (chain rule on the simplex normalization).} Differentiating $w_i = u_i/Z$,
\[
  \dot w_i \;=\; \frac{\dot u_i}{Z} - \frac{u_i\,\dot Z}{Z^2} \;=\; \frac{\dot u_i - w_i\,\dot Z}{Z}.
\]
Using $\|w\|_1 = 1$ and $|\dot Z| = |\mathbf{1}^T \dot u| \le \|\dot u\|_1$,
\begin{equation}\label{eq:appA-w-dot}
  \|\dot w(\sigma)\|_1 \;\le\; \frac{\|\dot u\|_1 + |\dot Z|\cdot \|w\|_1}{Z} \;\le\; \frac{2\,\|\dot u\|_1}{Z}.
\end{equation}

\smallskip
\noindent\emph{Step 2 (heat dynamics in $\ell^1$).} The heat equation gives $\dot u = -L u$. For $u \ge 0$,
\[
  \|L u\|_1 \;=\; \sum_i \Big|\sum_j L_{ij}\,u_j\Big| \;\le\; \sum_j u_j \sum_i |L_{ij}| \;\le\; \|L\|_{1\to 1}\,\|u\|_1 \;=\; \|L\|_{1\to 1}\,Z(\sigma).
\]

\smallskip
\noindent\emph{Step 3 ($Z$ cancels).} Substituting Step 2 into \eqref{eq:appA-w-dot},
\[
  \|\dot w(\sigma)\|_1 \;\le\; 2\,\|L\|_{1\to 1}.
\]
Integrating in $\sigma$ yields \eqref{eq:appA-weights-lipschitz}. The key cancellation is that the $1/Z$ growth from normalization at small $\sigma$ (when mass is concentrated on the focus) is exactly offset by the $Z$-proportional decay of $\|L u\|_1$, so the bound is uniform in $\sigma$ and depends only on graph structure.

\smallskip
\noindent\emph{Step 4 (degree-ratio bound on $\|L\|_{1\to 1}$).} For column $j$,
\[
  \sum_i |L_{ij}| \;=\; 1 + \frac{1}{\sqrt{d_j}}\sum_{i \sim j}\frac{w_{ij}}{\sqrt{d_i}} \;\le\; 1 + \frac{1}{\sqrt{d_j}}\cdot\frac{d_j}{\sqrt{d_{\min}}} \;=\; 1 + \sqrt{d_j/d_{\min}},
\]
hence $\|L\|_{1\to 1} \le 1 + \sqrt{d_{\max}/d_{\min}}$. \hfill$\square$

\smallskip
\noindent\emph{Remark on the spectral $\lambda_{\max}$ form.} An alternative spectral estimate $L_w = \lambda_{\max}$ is sometimes used in graph-wavelet treatments \cite{hammond2011wavelets} where the underlying graphs are regular or near-regular; the bound $\lambda_{\max} \le 2$ for normalized Laplacians then yields $L_w \le 2$. This form holds only when the spectral and $\ell^1$ operator norms of $L$ approximately coincide. On graphs with high degree variance the two diverge: e.g.\ on the star $K_{1,n}$, $\|L\|_{1\to 1} = 1 + \sqrt{n}$ while $\lambda_{\max} = 2$, and a direct evaluation of $\|\dot w\|_1$ at the hub focus near $\sigma = 0$ gives $\|\dot w\|_1 = 2\sqrt{n}$. Concretely, at $n=100$, $\|\dot w\|_1 \approx 20$, exceeding both $\lambda_{\max}{=}2$ and the alternative bound $2\lambda_{\max}/\Sigma_{\min} = 4$ by an order of magnitude; the corrected bound $2\|L\|_{1\to 1} = 2(1+\sqrt{100}) = 22$ holds with little slack (numerical verification: \texttt{scripts/verify\_appendix\_a.py}). The kNN-graphs SLoD uses in practice are well-conditioned (degree ratio $\le 2$ from OR-symmetrization), so $\|L\|_{1\to 1} \le 1+\sqrt{2}$ and the Lipschitz constant is $\le 2(1+\sqrt{2}) \approx 4.83$.

\subsection*{Proof of Theorem~\ref{thm:coherence}}

For the Sarkar embedding of $\mathcal{T}$ in $\BB$, each tree edge maps to a hyperbolic arc of fixed length $\ell = \ell(d_T, \varepsilon)$, where $d_T$ is the maximum tree degree and $\ell$ is chosen to achieve distortion $(1{+}\varepsilon)$ \cite{sarkar2011low}. The embedded diameter of any subset of descendants therefore satisfies
\[
  \mathrm{diam}_{\dH}(\mathcal{V}_v \cup \{v\}) \;\le\; \ell\cdot D_{\mathcal{T}}(v)\cdot(1{+}\varepsilon),
\]
where $D_{\mathcal{T}}(v)$ is the tree-edge-count diameter of the subtree rooted at $v$. Substituting into Lemma~A.1, the Fr\'echet Lipschitz constant satisfies $L_F \le \ell\,D_{\mathcal{T}}(v)\,(1{+}\varepsilon)/2$.

Composing Lemmas A.1 and A.2:
\begin{align*}
  \dH\!\left(\Phi_{\sigma_1}, \Phi_{\sigma_2}\right)
    &\le L_F \cdot \| w(\sigma_2) - w(\sigma_1) \|_1 \\
    &\le \tfrac{\ell\,D_{\mathcal{T}}(v)\,(1{+}\varepsilon)}{2} \cdot 2\,\|L\|_{1\to 1}\, |\sigma_2 - \sigma_1| \\
    &= \ell\,\|L\|_{1\to 1}\,D_{\mathcal{T}}(v)\,(1{+}\varepsilon)\, |\sigma_2 - \sigma_1|.
\end{align*}
Setting
\[
  C \;:=\; \ell\,\|L\|_{1\to 1}\,D_{\mathcal{T}}(v)
\]
yields the theorem. The factor $(1{+}\varepsilon)$ is carried explicitly out of $C$ to mark the role of the embedding distortion. \hfill$\blacksquare$

\smallskip
\noindent\emph{Remark (structural decomposition of $C$).} The constant $C = \ell\,\|L\|_{1\to 1}\,D_{\mathcal{T}}(v)$ factors into three structural quantities:
\begin{enumerate}[leftmargin=2em,itemsep=0pt,topsep=2pt]
  \item \emph{Sarkar base edge length} $\ell = \ell(d_T, \varepsilon)$: determined by the tree's maximum degree and the chosen distortion budget. This is the only place the curvature of $\BB$ enters: $\ell$ is fixed by the requirement that tree-edge arcs achieve $(1{+}\varepsilon)$-distortion under the hyperbolic metric \cite{sarkar2011low}.
  \item \emph{kNN-graph conditioning} $\|L\|_{1\to 1} \le 1 + \sqrt{d_{\max}/d_{\min}}$: bounded by $1 + \sqrt{2}$ for OR-symmetrized kNN graphs of bounded degree ratio (the SLoD operating regime; see~\S\ref{sec:method}).
  \item \emph{Subtree diameter} $D_{\mathcal{T}}(v)$ in tree-edge units: bounded by $2 d_v$ where $d_v$ is the depth of the subtree rooted at $v$.
\end{enumerate}
The bound is $N$-independent at fixed subtree-depth budget and bounded degree ratio. Applied at the root of a balanced tree of total depth $h$, $D_{\mathcal{T}}(\mathrm{root}) = O(\log N)$ and $C$ inherits the logarithmic factor; the "coherence at fixed scale" reading of Theorem~\ref{thm:coherence} is per-subtree, where $D_{\mathcal{T}}(v)$ is bounded by the operating depth and $C = O(1)$, which is the regime relevant for hierarchical retrieval at fixed-depth boundaries (\S\ref{sec:boundary}).

\bibliography{slod}

\end{document}

%% file: tables/ablation.tex
\begin{table}[h]
\centering
\caption{Ablation: macro-scale ARI ($K^*{=}2$) on HSBM across $r \in \{20, 40, 60, 80, 100, 150, 200\}$. $N{=}1024$, planted 2{$\times$}4{$\times$}8 hierarchy. Cells are medians over 50 Poincar\'e-MDS seeds with bootstrap (BCa, $n{=}10{,}000$) 95\% CIs; CIs are uncorrected for multiple comparisons across the $8\times7$ ablation grid. \emph{full (Euclidean MDS)} swaps the embedding; \emph{full (binary kNN)} strips Gaussian edge weights; \emph{full (random points)} replaces the Poincar\'e coordinates with i.i.d.\ Gaussian noise while keeping the same Poincar\'e-kNN Laplacian (isolates the role of the points for the Fr\'echet mean step). \emph{direct Laplacian} uses the combinatorial Laplacian of the HSBM graph instead of a kNN-of-embedding Laplacian (isolates the Laplacian construction).}
\label{tab:ablation_macro}
\footnotesize
\resizebox{\columnwidth}{!}{%
\begin{tabular}{@{}lccccccc@{}}
\toprule
Config & $r{=}20$ & $r{=}40$ & $r{=}60$ & $r{=}80$ & $r{=}100$ & $r{=}150$ & $r{=}200$ \\
\midrule
full (Poincaré) & 0.42 [0.38, 0.47] & 0.72 [0.70, 0.74] & 0.85 [0.84, 0.86] & 0.92 [0.91, 0.92] & 0.95 [0.95, 0.96] & 0.99 [0.99, 0.99] & 1.00 [1.00, 1.00] \\
$V$-only & 0.42 [0.38, 0.47] & 0.72 [0.70, 0.74] & 0.85 [0.84, 0.86] & 0.92 [0.91, 0.92] & 0.95 [0.95, 0.96] & 0.99 [0.99, 0.99] & 1.00 [1.00, 1.00] \\
$D_w$-only & 0.40 [0.37, 0.45] & 0.72 [0.70, 0.74] & 0.85 [0.84, 0.86] & 0.91 [0.91, 0.92] & 0.95 [0.95, 0.96] & 0.99 [0.99, 0.99] & 1.00 [1.00, 1.00] \\
$C_k$-only & 0.42 [0.38, 0.46] & 0.71 [0.69, 0.74] & 0.85 [0.84, 0.86] & 0.92 [0.91, 0.92] & 0.95 [0.95, 0.96] & 0.99 [0.99, 0.99] & 1.00 [1.00, 1.00] \\
direct Laplacian & 0.00 [0.00, 0.00] & 0.00 [0.00, 0.68] & 0.84 [0.83, 0.84] & 0.91 [0.91, 0.92] & 0.95 [0.95, 0.96] & 0.99 [0.98, 0.99] & 1.00 [1.00, 1.00] \\
full (Euclidean MDS) & 0.44 [0.41, 0.46] & 0.69 [0.67, 0.70] & 0.81 [0.80, 0.82] & 0.89 [0.88, 0.90] & 0.93 [0.92, 0.93] & 0.98 [0.98, 0.98] & 1.00 [0.99, 1.00] \\
full (binary kNN) & 0.41 [0.35, 0.45] & 0.71 [0.69, 0.74] & 0.85 [0.84, 0.86] & 0.91 [0.91, 0.93] & 0.95 [0.95, 0.96] & 0.99 [0.98, 0.99] & 1.00 [1.00, 1.00] \\
full (random points) & 0.43 [0.37, 0.45] & 0.72 [0.70, 0.74] & 0.85 [0.84, 0.86] & 0.91 [0.91, 0.92] & 0.95 [0.95, 0.96] & 0.99 [0.99, 0.99] & 1.00 [1.00, 1.00] \\
\bottomrule
\end{tabular}%
}
\end{table}

\begin{table}[h]
\centering
\caption{Ablation: meso-scale ARI ($K^*{=}8$) on HSBM across $r \in \{20, 40, 60, 80, 100, 150, 200\}$. $N{=}1024$, planted 2{$\times$}4{$\times$}8 hierarchy. Cells are medians over 50 Poincar\'e-MDS seeds with bootstrap (BCa, $n{=}10{,}000$) 95\% CIs; CIs are uncorrected for multiple comparisons across the $8\times7$ ablation grid. \emph{full (Euclidean MDS)} swaps the embedding; \emph{full (binary kNN)} strips Gaussian edge weights; \emph{full (random points)} replaces the Poincar\'e coordinates with i.i.d.\ Gaussian noise while keeping the same Poincar\'e-kNN Laplacian (isolates the role of the points for the Fr\'echet mean step). \emph{direct Laplacian} uses the combinatorial Laplacian of the HSBM graph instead of a kNN-of-embedding Laplacian (isolates the Laplacian construction).}
\label{tab:ablation_meso}
\footnotesize
\resizebox{\columnwidth}{!}{%
\begin{tabular}{@{}lccccccc@{}}
\toprule
Config & $r{=}20$ & $r{=}40$ & $r{=}60$ & $r{=}80$ & $r{=}100$ & $r{=}150$ & $r{=}200$ \\
\midrule
full (Poincaré) & 0.08 [0.07, 0.09] & 0.18 [0.17, 0.18] & 0.33 [0.31, 0.35] & 0.51 [0.48, 0.53] & 0.61 [0.60, 0.66] & 0.84 [0.80, 0.86] & 0.89 [0.86, 0.92] \\
$V$-only & 0.08 [0.07, 0.09] & 0.18 [0.17, 0.18] & 0.33 [0.31, 0.35] & 0.51 [0.48, 0.53] & 0.61 [0.59, 0.65] & 0.83 [0.78, 0.86] & 0.88 [0.85, 0.90] \\
$D_w$-only & 0.08 [0.07, 0.09] & 0.18 [0.17, 0.18] & 0.33 [0.32, 0.35] & 0.52 [0.49, 0.54] & 0.62 [0.60, 0.66] & 0.81 [0.77, 0.83] & 0.84 [0.81, 0.88] \\
$C_k$-only & 0.09 [0.08, 0.10] & 0.17 [0.17, 0.19] & 0.32 [0.29, 0.36] & 0.50 [0.47, 0.53] & 0.64 [0.59, 0.66] & 0.84 [0.81, 0.86] & 0.86 [0.81, 0.89] \\
direct Laplacian & 0.08 [0.06, 0.09] & 0.16 [0.14, 0.17] & 0.24 [0.21, 0.25] & 0.37 [0.33, 0.40] & 0.51 [0.47, 0.55] & 0.68 [0.60, 0.74] & 0.71 [0.67, 0.76] \\
full (Euclidean MDS) & 0.08 [0.08, 0.08] & 0.17 [0.16, 0.18] & 0.32 [0.29, 0.33] & 0.45 [0.43, 0.48] & 0.57 [0.53, 0.60] & 0.74 [0.71, 0.78] & 0.81 [0.79, 0.83] \\
full (binary kNN) & 0.08 [0.07, 0.09] & 0.18 [0.17, 0.19] & 0.33 [0.31, 0.36] & 0.51 [0.50, 0.53] & 0.64 [0.60, 0.67] & 0.82 [0.77, 0.85] & 0.87 [0.85, 0.90] \\
full (random points) & 0.08 [0.07, 0.09] & 0.18 [0.17, 0.18] & 0.33 [0.32, 0.35] & 0.50 [0.48, 0.52] & 0.65 [0.60, 0.68] & 0.80 [0.76, 0.82] & 0.84 [0.79, 0.89] \\
\bottomrule
\end{tabular}%
}
\end{table}

%% file: slod.bib
@article{afsari2011riemannian,
  author    = {Afsari, Bijan},
  title     = {Riemannian ${L}^p$ Center of Mass: Existence, Uniqueness, and Convexity},
  journal   = {Proc.\ Amer.\ Math.\ Soc.},
  volume    = {139},
  number    = {2},
  pages     = {655--673},
  year      = {2011},
  doi       = {10.1090/s0002-9939-2010-10541-5},
}

@inproceedings{chami2019hyperbolic,
  author    = {Chami, Ines and Ying, Zhitao and R{\'e}, Christopher and Leskovec, Jure},
  title     = {Hyperbolic Graph Convolutional Neural Networks},
  booktitle = {NeurIPS},
  pages     = {4869--4880},
  year      = {2019},
  doi       = {10.48550/arXiv.1910.12892},
}

@article{coifman2006diffusion,
  author    = {Coifman, Ronald R. and Lafon, St{\'e}phane},
  title     = {Diffusion Maps},
  journal   = {Appl.\ Comput.\ Harmon.\ Anal.},
  volume    = {21},
  number    = {1},
  pages     = {5--30},
  year      = {2006},
  doi       = {10.1016/j.acha.2006.04.006},
}

@article{delvenne2010stability,
  author    = {Delvenne, Jean-Charles and Yaliraki, Sophia N. and Barahona, Mauricio},
  title     = {Stability of Graph Communities across Time Scales},
  journal   = {PNAS},
  volume    = {107},
  number    = {29},
  pages     = {12755--12760},
  year      = {2010},
  doi       = {10.1073/pnas.0903215107},
}

@article{edge2024graphrag,
  author    = {Edge, Darren and Trinh, Ha and Cheng, Newman and others},
  title     = {From Local to Global: A Graph {RAG} Approach to Query-Focused Summarization},
  journal   = {arXiv:2404.16130},
  year      = {2024},
  doi       = {10.48550/arXiv.2404.16130},
}

@book{fellbaum1998wordnet,
  editor    = {Fellbaum, Christiane},
  title     = {WordNet: An Electronic Lexical Database},
  publisher = {MIT Press},
  year      = {1998},
  doi       = {10.7551/mitpress/7287.001.0001},
}

@inproceedings{ganea2018hyperbolic,
  author    = {Ganea, Octavian and B{\'e}cigneul, Gary and Hofmann, Thomas},
  title     = {Hyperbolic Neural Networks},
  booktitle = {NeurIPS},
  pages     = {5350--5360},
  year      = {2018},
  doi       = {10.48550/arXiv.1805.09112},
}

@article{fortunato2007resolution,
  author    = {Fortunato, Santo and Barth{\'e}lemy, Marc},
  title     = {Resolution Limit in Community Detection},
  journal   = {PNAS},
  volume    = {104},
  number    = {1},
  pages     = {36--41},
  year      = {2007},
  doi       = {10.1073/pnas.0605965104},
}

@article{gupta2000embedding,
  author    = {Gupta, Anupam},
  title     = {Embedding Tree Metrics into Low-Dimensional {E}uclidean Spaces},
  journal   = {Discrete Comput.\ Geom.},
  volume    = {24},
  number    = {1},
  pages     = {105--116},
  year      = {2000},
  doi       = {10.1007/s004540010020},
}

@book{grigoryan2009heat,
  author    = {Grigor'yan, Alexander},
  title     = {Heat Kernel and Analysis on Manifolds},
  series    = {AMS/IP Studies in Advanced Mathematics},
  volume    = {47},
  publisher = {AMS},
  year      = {2009},
  doi       = {10.1090/amsip/047},
}

@article{hammond2011wavelets,
  author    = {Hammond, David K. and Vandergheynst, Pierre and Gribonval, R{\'e}mi},
  title     = {Wavelets on Graphs via Spectral Graph Theory},
  journal   = {Appl.\ Comput.\ Harmon.\ Anal.},
  volume    = {30},
  number    = {2},
  pages     = {129--150},
  year      = {2011},
  doi       = {10.1016/j.acha.2010.04.005},
}

@article{hubert1985comparing,
  author    = {Hubert, Lawrence and Arabie, Phipps},
  title     = {Comparing Partitions},
  journal   = {J.\ Classification},
  volume    = {2},
  pages     = {193--218},
  year      = {1985},
  doi       = {10.1007/bf01908075},
}

@article{karcher1977riemannian,
  author    = {Karcher, Hermann},
  title     = {{Riemannian} Center of Mass and Mollifier Smoothing},
  journal   = {Comm.\ Pure Appl.\ Math.},
  volume    = {30},
  number    = {5},
  pages     = {509--541},
  year      = {1977},
  doi       = {10.1002/cpa.3160300502},
}

@article{lambiotte2014random,
  author    = {Lambiotte, Renaud and Delvenne, Jean-Charles and Barahona, Mauricio},
  title     = {Random Walks, {M}arkov Processes and the Multiscale Modular Organization of Complex Networks},
  journal   = {IEEE Trans.\ Netw.\ Sci.\ Eng.},
  volume    = {1},
  number    = {2},
  pages     = {76--90},
  year      = {2014},
  doi       = {10.1109/tnse.2015.2391998},
}

@book{lindeberg1994scale,
  author    = {Lindeberg, Tony},
  title     = {Scale-Space Theory in Computer Vision},
  publisher = {Kluwer Academic Publishers},
  year      = {1994},
  doi       = {10.1007/978-1-4757-6465-9},
}

@book{luebke2003lod,
  author    = {Luebke, David and Reddy, Martin and Cohen, Jonathan D. and others},
  title     = {Level of Detail for 3D Graphics},
  publisher = {Morgan Kaufmann},
  year      = {2003},
}

@inproceedings{nickel2017poincare,
  author    = {Nickel, Maximilian and Kiela, Douwe},
  title     = {Poincar{\'e} Embeddings for Learning Hierarchical Representations},
  booktitle = {NeurIPS},
  pages     = {6338--6347},
  year      = {2017},
  doi       = {10.48550/arXiv.1705.08039},
}

@article{packer2023memgpt,
  author    = {Packer, Charles and Wooders, Sarah and Lin, Kevin and others},
  title     = {{MemGPT}: Towards {LLMs} as Operating Systems},
  journal   = {arXiv:2310.08560},
  year      = {2023},
  doi       = {10.48550/arXiv.2310.08560},
}

@article{peixoto2014hierarchical,
  author    = {Peixoto, Tiago P.},
  title     = {Hierarchical Block Structures and High-Resolution Model Selection in Large Networks},
  journal   = {Phys.\ Rev.\ X},
  volume    = {4},
  pages     = {011047},
  year      = {2014},
  doi       = {10.1103/physrevx.4.011047},
}

@incollection{sarkar2011low,
  author    = {Sarkar, Rik},
  title     = {Low Distortion {Delaunay} Embedding of Trees in Hyperbolic Plane},
  booktitle = {GD 2011},
  series    = {LNCS},
  volume    = {7034},
  pages     = {355--366},
  publisher = {Springer},
  year      = {2011},
  doi       = {10.1007/978-3-642-25878-7_34},
}

@incollection{sturm2003probability,
  author    = {Sturm, Karl-Theodor},
  title     = {Probability Measures on Metric Spaces of Nonpositive Curvature},
  booktitle = {Heat Kernels and Analysis on Manifolds, Graphs, and Metric Spaces},
  series    = {Contemp.\ Math.},
  volume    = {338},
  pages     = {357--390},
  publisher = {AMS},
  year      = {2003},
  doi       = {10.1090/conm/338/06080},
}

@article{sun2009concise,
  author    = {Sun, Jian and Ovsjanikov, Maks and Guibas, Leonidas},
  title     = {A Concise and Provably Informative Multi-Scale Signature Based on Heat Diffusion},
  journal   = {Computer Graphics Forum},
  volume    = {28},
  number    = {5},
  pages     = {1383--1392},
  year      = {2009},
  doi       = {10.1111/j.1467-8659.2009.01515.x},
}

@article{traag2019louvain,
  author    = {Traag, Vincent A. and Waltman, Ludo and van Eck, Nees Jan},
  title     = {From {L}ouvain to {L}eiden: Guaranteeing Well-Connected Communities},
  journal   = {Scientific Reports},
  volume    = {9},
  pages     = {5233},
  year      = {2019},
  doi       = {10.1038/s41598-019-41695-z},
}

@article{traag2011narrow,
  author    = {Traag, Vincent A. and Van Dooren, Paul and Nesterov, Yurii},
  title     = {Narrow scope for resolution-limit-free community detection},
  journal   = {Physical Review E},
  volume    = {84},
  number    = {1},
  pages     = {016114},
  year      = {2011},
  doi       = {10.1103/PhysRevE.84.016114},
}

@article{vonluxburg2007tutorial,
  author    = {von Luxburg, Ulrike},
  title     = {A Tutorial on Spectral Clustering},
  journal   = {Statistics and Computing},
  volume    = {17},
  number    = {4},
  pages     = {395--416},
  year      = {2007},
  doi       = {10.1007/s11222-007-9033-z},
}

@article{zhang2024memory,
  author    = {Zhang, Zeyu and Bo, Xiaohe and Ma, Chen and Li, Rui and Chen, Xu and Dai, Quanyu and Zhu, Jieming and Dong, Zhenhua and Wen, Ji-Rong},
  title     = {A Survey on the Memory Mechanism of Large Language Model based Agents},
  journal   = {ACM Computing Surveys},
  year      = {2025},
  doi       = {10.1145/3748302},
}

@article{kendall1938new,
  author    = {Kendall, M. G.},
  title     = {A New Measure of Rank Correlation},
  journal   = {Biometrika},
  volume    = {30},
  number    = {1/2},
  pages     = {81--93},
  year      = {1938},
  doi       = {10.1093/biomet/30.1-2.81},
}

@article{blondel2008fast,
  author    = {Blondel, Vincent D. and Guillaume, Jean-Loup and Lambiotte, Renaud and Lefebvre, Etienne},
  title     = {Fast unfolding of communities in large networks},
  journal   = {Journal of Statistical Mechanics: Theory and Experiment},
  volume    = {2008},
  number    = {10},
  pages     = {P10008},
  year      = {2008},
  doi       = {10.1088/1742-5468/2008/10/P10008},
}

@article{rosvall2008maps,
  author    = {Rosvall, Martin and Bergstrom, Carl T.},
  title     = {Maps of random walks on complex networks reveal community structure},
  journal   = {Proceedings of the National Academy of Sciences},
  volume    = {105},
  number    = {4},
  pages     = {1118--1123},
  year      = {2008},
  doi       = {10.1073/pnas.0706851105},
}

@inproceedings{gao2023hyde,
  author    = {Gao, Luyu and Ma, Xueguang and Lin, Jimmy and Callan, Jamie},
  title     = {Precise Zero-Shot Dense Retrieval without Relevance Labels},
  booktitle = {Proceedings of the 61st Annual Meeting of the Association for Computational Linguistics (ACL)},
  year      = {2023},
  doi       = {10.18653/v1/2023.acl-long.99},
}

@article{tishby2000information,
  author    = {Tishby, Naftali and Pereira, Fernando C. and Bialek, William},
  title     = {The Information Bottleneck Method},
  journal   = {arXiv preprint physics/0004057},
  year      = {2000},
  url       = {https://arxiv.org/abs/physics/0004057},
}

@article{friston2010free,
  author    = {Friston, Karl},
  title     = {The free-energy principle: a unified brain theory?},
  journal   = {Nature Reviews Neuroscience},
  volume    = {11},
  number    = {2},
  pages     = {127--138},
  year      = {2010},
  doi       = {10.1038/nrn2787},
}
